\documentclass{article}

\usepackage{microtype}
\usepackage{graphicx}
\usepackage{booktabs} 

\usepackage{hyperref}
\usepackage{xspace}
\usepackage{tikz}
\usetikzlibrary {positioning}
\definecolor {processblue}{cmyk}{0.96,0,0,0}

 \usepackage[accepted]{icml2024}

\usepackage{amsmath}
\usepackage{amssymb}
\usepackage{mathtools}
\usepackage{amsthm}

\usepackage[capitalize,noabbrev]{cleveref}

\theoremstyle{plain}
\newtheorem{theorem}{Theorem}[section]

\newtheorem{lemma}[theorem]{Lemma}

\theoremstyle{definition}

\theoremstyle{remark}
\newtheorem{remark}[theorem]{Remark}

\usepackage[textsize=tiny]{todonotes}

\usepackage{hyperref,pifont}       
\usepackage{url}            
\usepackage{booktabs}       
\usepackage{amsfonts}       
\usepackage{nicefrac}       
\usepackage{microtype}      
\usepackage{xcolor,dsfont}         
\usepackage{algorithm}
\usepackage{algorithmic}
\usepackage{graphicx}
\usepackage{amssymb}
\usepackage{subcaption}
\usepackage{xspace}

\usepackage{amsmath,amsfonts,comment,bbm,setspace}
\usepackage{hyperref}
\usepackage{url}

\newcommand{\cA}{\mathcal{A}}
\newcommand{\cS}{\mathcal{S}}

\newcommand{\cE}{\mathcal{E}}

\def\cA{{\mathcal{A}}}   
\def\cE{{\mathcal{E}}}   
   
 \def\cN{{\mathcal{N}}} \def\cO{{\mathcal{O}}} \def\cP{{\mathcal{P}}}
  \def\cS{{\mathcal{S}}} 
   
 \def\cZ{{\mathcal{Z}}}

   \def\bs{{\mathbf{s}}}

\newcommand{\bef}{\begin{figure}}
\newcommand{\eef}{\end{figure}}
\newcommand{\beq}{\begin{eqnarray}}
\newcommand{\eeq}{\end{eqnarray}}

\newcommand{\armab}{\texttt{ARMAB}\xspace}
\newcommand{\ucarmab}{\texttt{UCMD-ARMAB}\xspace}

\icmltitlerunning{Provably Efficient Reinforcement Learning for Adversarial RMAB with Unknown Transitions and Bandit Feedback}

\begin{document}

\twocolumn[

\icmltitle{Provably Efficient Reinforcement Learning for Adversarial Restless Multi-Armed Bandits with Unknown Transitions and Bandit Feedback}




\begin{icmlauthorlist}
\icmlauthor{Guojun Xiong}{yyy}
\icmlauthor{Jian Li}{yyy}
\end{icmlauthorlist}

\icmlaffiliation{yyy}{Stony Brook University}

\icmlcorrespondingauthor{Guojun Xiong}{guojun.xiong@stonybrook.edu}
\icmlcorrespondingauthor{Jian Li}{jian.li.3@stonybrook.edu}

\icmlkeywords{Machine Learning, ICML}

\vskip 0.3in
]



\printAffiliationsAndNotice{}  

\begin{abstract}
Restless multi-armed bandits (RMAB) play a central role in modeling sequential decision making problems under an instantaneous activation constraint that at most $B$ arms can be activated at any decision epoch. Each restless arm is endowed with a state that evolves independently according to a Markov decision process regardless of being activated or not. In this paper, we consider the task of learning in episodic RMAB with unknown transition functions and adversarial rewards, which can change arbitrarily across episodes. Further, we consider a challenging but natural bandit feedback setting that only adversarial rewards of activated arms are revealed to the decision maker (DM). The goal of the DM is to maximize its total adversarial rewards during the learning process while the instantaneous activation constraint must be satisfied in each decision epoch. We develop a novel reinforcement learning algorithm with two key contributors: a novel biased adversarial reward estimator to deal with bandit feedback and unknown transitions, and a low-complexity index policy to satisfy the instantaneous activation constraint. We show $\tilde{\mathcal{O}}(H\sqrt{T})$ regret bound for our algorithm, where $T$ is the number of episodes and $H$ is the episode length. To our best knowledge, this is the first algorithm to ensure  $\tilde{\mathcal{O}}(\sqrt{T})$ regret for adversarial RMAB in our considered challenging settings.

\end{abstract}

\section{Introduction}\label{sec:intro}

Restless multi-armed bandits (RMAB) \citep{whittle1988restless} has been widely used to study sequential decision making problems with an \textit{``instantaneous activation constraint''} that at most $B$ out of $N$ arms can be activated at any decision epoch, ranging from wireless scheduling \citep{sheng2014data,cohen2014restless,xiong2022index}, 
resource allocation \citep{glazebrook2011general,larranaga2014index,xiong2023reinforcement,xiong2022whittle}, 
to healthcare \citep{killian2021beyond}.  
Each arm is described by a Markov decision process (MDP) \citep{puterman1994markov}, and evolves stochastically according to two different transition kernels, depending on whether the arm is activated or not. Rewards are generated with each transition. The goal of the decision maker (DM) is to maximize the total expected reward over a finite-horizon \citep{xiong2022AAAI} or an infinite-horizon \citep{wang2020restless,avrachenkov2022whittle,xiong2022Neurips,xiong2023finite,wang2024online} under the instantaneous activation constraint.

The majority of the literature on RMAB consider a stochastic environment, where both the rewards and dynamics of the environments are assumed to be stationary over time. However, in real-world applications such as online advertising and revenue management, rewards\footnote{Although most existing literature on adversarial learning use the term ``loss'' instead of ``reward'', we choose to use the latter, or more specifically ``adversarial reward''
in this paper to be consistent with RMAB literature. One can translate between rewards and losses by taking negation.} are not necessarily stationary but can change arbitrarily between episodes \citep{lee2023online}. To this end, we study the problem of learning \textit{a finite-horizon adversarial RMAB (\armab) with unknown transitions} over $T$ episodes. In each episode, all arms start from a fixed initial state, and the DM repeats the followings for a fixed number of $H$ decision epochs: determine whether or not to activate each arm while the instantaneous activation constraint must be satisfied, receive adversarial rewards from each arm, which transits to the next state according to some unknown transition functions. Specifically, we consider a challenging but natural \textit{bandit feedback} setting, where the adversarial rewards in each decision epoch are only revealed to the DM when the state-action pairs are visited. The goal of the DM is to minimize its regret, which is the difference between its total adversarial rewards and the total rewards received by an optimal fixed policy.

\begin{table*}[t]
\centering
\caption{Comparison with existing works, where $T$ is the number of episodes and $H$ is the length of each episode. }
\scalebox{0.85}{
\begin{tabular}{|c|c|c|c|c|c|c|}
\hline
Paper & Model & Setting &Feedback & Constraint & Algorithm & Regret \\\hline
\hline \cite{rosenberg2019online}& MDP & Adversarial& {Full} & \ding{55} & OMD & $\tilde{\mathcal{O}}(H\sqrt{T})$\\ 
\cite{rosenberg2019b} & MDP & Adversarial& {Bandi}t & \ding{55}& OMD & $\tilde{\mathcal{O}}(H^{2/3}T^{3/4})$\\
\cite{jin2020learning} & MDP &  Adversarial&{Bandit} & \ding{55} & OMD & $\tilde{\mathcal{O}}(H\sqrt{T})$ \\ 
\cite{luo2021policy} & MDP & Adversarial& {Bandit} & \ding{55} & Policy Optimization & $\tilde{\mathcal{O}}(H^2\sqrt{T})$\\
\hline
\cite{qiu2020upper} & CMDP & Adversarial&{Full} & {Average} & Primal-dual& $\tilde{\mathcal{O}}(H\sqrt{T})$\\
\cite{germano2023best} & CMDP & Adversarial& {Full} & {Average} & Primal-dual& $\tilde{\mathcal{O}}(HT^{3/4})$\\
\hline
\cite{wang2020restless} & RMAB & Stochastic&{Full} & Hard & Generative model& $\tilde{\mathcal{O}}(H^{2/3}T^{2/3})$\\
\cite{xiong2022AAAI,xiong2022Neurips} & RMAB & Stochastic&{Full} & Hard & Index-based& $\tilde{\mathcal{O}}(\sqrt{HT})$\\ \hline
\textbf{This Work} & \textbf{RMAB} & \textbf{Adversarial}& \textbf{Bandit} & \textbf{Hard} & \textbf{Index-based OMD}&${\tilde{\mathcal{O}}(H\sqrt{T})}$\\
\hline
\end{tabular}}
\label{table:1}
\end{table*}

To achieve this goal, we develop an episodic reinforcement learning (RL) algorithm named \ucarmab. First, to handle unknown transitions of each arm, we construct confidence sets to guarantee the true ones lie in these sets with high probability (Section~\ref{sec:confidence-set}). Second, to handle adversarial rewards, we apply Online Mirror Descent (OMD) to solve a relaxed problem, rather than directly on \armab, in terms of occupancy measures (Section~\ref{sec:OMD}). This is due to the fact that \armab is known to be computationally intractable even in the offline setting \citep{papadimitriou1994complexity}. We note that OMD has also been used in  adversarial MDP \citep{rosenberg2019online, jin2020learning} and CMDP \citep{qiu2020upper}. However, they considered a stationary occupancy measure due to the existing of stationary policies, which is not the case for our finite-horizon RMAB with instantaneous activation constraint. This requires us to leverage a time-dependent occupancy measure.

Third, a key difference compared to stochastic RMAB \citep{wang2020restless,xiong2022AAAI,xiong2022Neurips} is that with bandit feedback and to apply the above OMD, we must construct adversarial reward estimators since the rewards of arms are not completely revealed to the DM. We address this challenge by developing a novel biased overestimated reward estimator (Section~\ref{sec:estimator}) based on the observations of counts for each state-action pairs.  Finally, to handle the instantaneous activation constraint in \armab, we develop a low-complexity index policy (Section~\ref{sec:index}) based on the solutions from the OMD in Section~\ref{sec:OMD}. This is another key difference compared to adversarial MDP or CMDP, which requires us to explicitly characterize the regret due to the implementation of such an index policy.

We prove that \ucarmab achieves $\tilde{\mathcal{O}}(H\sqrt{T})$ regret, where $T$ is the number of episodes and $H$ is the episode length. Although our regret bound exhibits a gap (i.e., $\sqrt{H}$ times larger) to that of stochastic RMAB \citep{xiong2022AAAI, xiong2022Neurips}, to our best knowledge, our result is the first to achieve $\tilde{\mathcal{O}}(\sqrt{T})$ regret for adversarial RMAB with bandit feedback and unknown transition functions, a harder problem compared to stochastic RMABs.

\textbf{Notations.}  We use calligraphy letter $\mathcal{A}$ to denote a finite set with cardinality $|\mathcal{A}|$, and $[N]$ to denote the set of integers $\{1,\cdots, N\}$.

\section{Related Work}
We discuss our related work from three categories: MDP, CMDP and RMAB. In particular, we mainly focus on the adversarial settings for the former two. 

\textbf{Adversarial MDP.} \citet{even2009online} proposed the adversarial MDP model with arbitrarily changed loss functions and a fixed stochastic transition function. The first to consider unknown transition function with full information feedback is \citet{neu2012adversarial}, which proposed a Follow-the-Perturbed-Optimistic algorithm with an $\tilde{\mathcal{O}}(H|\cS||\cA|\sqrt{T})$ regret.  \citet{rosenberg2019online} improved the bound to $\tilde{\mathcal{O}}(H|\cS|\sqrt{|\cA|T})$ through a UCRL2-based online optimization algorithm.  A more challenging bandit feedback setting was considered in \citet{rosenberg2019b}, which achieves an $\tilde{\mathcal{O}}(H^{3/2}|\cS||\cA|^{1/4}T^{3/4})$ regret. \citet{jin2020learning} further achieves an improved $\tilde{\mathcal{O}}(H|\cS|\sqrt{|\cA|T})$ regret via a novel reward estimator and a modified radius of upper confidence ball. Under a similar setting, a policy optimization method is developed in \citet{luo2021policy} with $\tilde{\mathcal{O}}(H^2|\cS|\sqrt{|\cA|T})$ regret. 
Another line of recent works further consider the settings of linear function approximation \citep{neu2021online}, the best-of-both-world \citep{jin2021best}, delay bandit feedback \citep{jin2022near} and adversarial transition functions \citep{jin2023no}.

\textbf{Adversarial Constrained MDP (CMDP).} CMDP \citep{altman1999constrained} plays an important role in control and planning, which aims to maximize a reward over all available policies subject to constraints that enforce the fairness or safety of the policies. \citet{qiu2020upper} is one of the first to study CMDP with adversarial losses and unknown transition function. A primal-dual algorithm was proposed with $\tilde{\mathcal{O}}(H|\cS|\sqrt{|\cA|T})$ regret and constraint violation. \citet{germano2023best} further considered both adversarial losses and constraints, and proposed a best-of-both-world algorithm, which achieves $\tilde{\mathcal{O}}(HT^{3/4})$ regret and constraint violation.

\textbf{RMAB.} RMAB was first introduced in \citet{whittle1988restless}, and has been widely studied, see \citet{nino2023markovian} and references therein. In particular, RL algorithms have been proposed for RMAB with unknown transitions. Colored-UCRL2 is the state-of-the-art method for {online RMAB with $\tilde{\mathcal{O}}(\sqrt{HT})$ regret. To address the exponential computational complexity of colored-UCRL2, low-complexity RL algorithms have been developed. For example, \citet{wang2020restless} proposed a generative model-based algorithm with $\tilde{\mathcal{O}}(H^{2/3}T^{2/3})$ regret, and \citet{xiong2022AAAI,xiong2022Neurips} designed index-aware RL algorithms for both finite-horizon and infinite-horizon average reward settings with $\tilde{\mathcal{O}}(\sqrt{HT})$ regret\footnote{For a fair comparison, the total time horizon for stochastic RMAB \citep{wang2020restless,xiong2022AAAI,xiong2022Neurips} is set to be $HT$.}. However, most of the existing literature on RMAB focus on the stochastic setting, where the reward functions are stochastically sampled from a fixed distribution, either known or unknown. To our best knowledge, this work is the first to study RMAB in adversarial settings with unknown transition functions and bandit feedback.

\section{Model and Problem Formulation}\label{sec:model}

We formally define the adversarial RMAB, and introduce the online settings considered in this paper.

\subsection{\armab: Adversarial RMAB}

Consider an episodic adversarial RMAB with $N$ arms. Each arm $n\in[N]$ is associated with a  unichain MDP denoted by a tuple $(\cS, \cA, P_n, \{r_n^t, \forall t\in[T]\}, H)$. $\cS$ is a finite state space, and $\cA:=\{0,1\}$ is the set of binary actions. Using the standard terminology from RMAB literature, an arm is \textit{passive} when action $a=0$ is applied to it, and \textit{active} otherwise. $P_n:\cS\times\cA\times\cS\mapsto[0,1]$ is the transition kernel with $P_n(s^\prime|s,a)$ being the probability of transition to state $s^\prime$ from state $s$ by taking action $a$. $T$ is the number of episodes, each of which consists of $H$ decision epochs.  $r_n^t: \cS\times\cA\mapsto [0,1]$ is the adversarial reward function in episode $t$. For simplicity, let $r_n(s,0)=0, \forall s\in\cS, n\in[N]$.  We do not make any statistical assumption on the adversarial reward functions, \textit{which can be chosen arbitrarily}\footnote{
However, in stochastic RMAB, rewards follow a stochastic distribution, which is fixed between episodes. This leads to a different objective in adversarial RMAB as the DM must adapt to dynamic and potentially hostile reward structures while striving to find an optimal policy across all episodes.}. 

At decision epoch $h\in[H]$ in episode $t\in[T]$, an arm can be either active or passive. A policy determines what action to apply to each arm at $h$ under the instantaneous activation constraint that at most $B$ arms can be activated. Denote such a feasible policy in episode $t$ as $\pi^t$, and let $\{(S_n^{t,h}, A_n^{t,h})\in\cS\times \cA, \forall t\in[T], h\in[H], n\in[N]\}$ be the random tuple generated according to transition functions $\{P_n, \forall n\in[N]\}$ and $\pi^t$. 
The corresponding expected reward in episode $t$ is 
\begin{align}\label{eq:reward-ept}
   R_t (\pi^t)\!:= \mathbb{E}\left[\sum_{h=1}^{H}\sum_{n=1}^N \!r_n^t(S_n^{t,h},A_n^{t,h})\Big|\{P_n, \!\forall n\}, \pi^t\right].
\end{align}
The DM's goal is to find a policy $\pi$ for all $T$ episodes which maximizes the total expected adversarial reward under the instantaneous active constraint, i.e., 
\begin{align}
\armab: \max_{\pi}~ &R(T, \pi):=\sum_{t=1}^T  R_t(\pi)\allowdisplaybreaks\nonumber\\ 
\text{s.t.}~ & \sum_{n=1}^N A_n^{t,h}\leq B, \forall h\in [H] ,t\in [T].
    \label{eq:orginal_P}
\end{align}
It is known that even when the transition kernel of each arm and the adversarial reward functions in each episode are revealed to the DM at the very beginning, {finding an optimal policy over $T$ episodes, denoted as $\pi^{opt}$ for \armab~(\ref{eq:orginal_P})} is PSPACE-hard \citep{papadimitriou1994complexity}. The fundamental challenge lies in the explosion of state space and the curse of dimensionality prevents computing optimal policies. Exacerbating this challenge is the fact that transition kernels are often unknown in practice, and {adversarial reward functions are only revealed to the DM at the end of each episode in adversarial settings (see Section~\ref{sec:online-setting}).}

\begin{remark}
Most existing works on adversarial MDP \citep{rosenberg2019online,rosenberg2019b, jin2020learning} and CMDP  \citep{qiu2020upper, germano2023best} assume that the state space is loop-free. In other words, the state space $\cS$ can be divided into $L$ distinct layers, i.e., $\cS:=\cS_1\cup\ldots\cup\cS_L$ with a singleton initial layer $\cS_1=\{s_1\}$, a terminal layer $\cS_L=\{s_L\}$, and $\cS_\ell \cap \cS_j=\emptyset, j\neq \ell.$ Transitions only occur between consecutive layers, i.e., $P(s^\prime|s,a)>0$ if $s^\prime \in \cS_{\ell+1}, s\in\cS_\ell, \forall \ell\in[L]$. 
On one hand, many practical problems do not have a loop-free MDP. On the other hand, this assumption requires any non-loop-free MDP to extend its state space $L$ times to be transformed into a loop-free MDP with a fixed length $L$. This often enlarges the regret bound at least $L$ times. In this paper, we consider a general MDP  without such a restrictive loop-free assumption. 
\end{remark}

\subsection{Online Setting and Learning Regret}\label{sec:online-setting}

\begin{algorithm}[t]
	\caption{{Online Interactions between the DM and the Adversarial RMAB Environment}}
	\label{alg:importance-policy}
	\begin{algorithmic}[1]
		\REQUIRE State space $\cS$, action space $\cA$, and unknown transition functions $\{\cP_n, \forall n\}$;
  \FOR{$t=1$ to $T$}
\STATE {All arms start in state $S_n^{0}, \forall n$;}
\STATE Adversary decides the reward function $\{r_n^t, \forall n\}$, and the DM decides a policy $\pi^t$; 
  \FOR{$h=1$ to $H$}
  \STATE DM chooses actions $\{A_{n}^{t,h}, \forall n\}$ under the instantaneous activation constraint; observes adversarial reward $\{r_n^t(S_n^{t,h}, A_n^{t,h}),\forall n\}$; \STATE Arm $n, \forall n$ moves to the next state $S_n^{t,h+1}\sim P_n(\cdot|S_n^{t,h},A_n^{t,h})$ ;
  \STATE DM observes states $\{S_n^{t,h+1}, \forall n\}$.
  \ENDFOR
	 \ENDFOR 
	\end{algorithmic}
\end{algorithm}

We focus on the online adversarial settings where the underlying MDPs are unknown to the DM, and the adversarial reward is of bandit-feedback\footnote{We use the term ``bandit-feedback'' as in \citet{rosenberg2019b, jin2020learning} to denote that only the adversarial rewards of visited state-action pairs are revealed to the DM. }.
The interaction between the DM and the \armab environment is presented in Algorithm~\ref{alg:importance-policy}. Only the state and action spaces are known to the DM in advance, and the interaction proceeds in $T$ episodes. At the beginning of episode $t$, the adversary determines the adversarial reward functions, each arm $n$ starts from a fixed state $S_n^0, \forall n\in[N]$, and the DM determines a policy $\pi^t$ and then executes this policy for each decision epoch $h\in[H]$ in this episode. Specifically, at decision epoch $h$, the DM chooses actions $\{A_{n}^{t,h}, \forall n\in \cN\}$ for each arm according to $\pi^t(\cdot|S_n^{t,h})$ under the instantaneous activation constraint, i.e., $\sum_{n=1}^N A_{n}^{t,h}=B$, and each arm then moves to the next state $S_n^{t,h+1}$ sampled from $P_n(\cdot|S_n^{t,h},A_n^{t,h})$. The DM records the trajectory of the current episode $t$ and the adversarial rewards in each decision epoch are \textit{only revealed to the DM when the state-action pairs are visited due to bandit-feedback.}

The DM's goal is to minimize its regret, as defined by
\begin{align}\label{eq:regret}
    \Delta(T):=R(T, \pi^{opt})-\sum_{t=1}^T R_t(\pi^t), 
\end{align}
where $R(T, \pi^{opt})$ is the total expected adversarial rewards under the offline optimal policy $\pi^{opt}$ by solving~(\ref{eq:orginal_P}), and $R_t(\pi^t)$ is defined in~(\ref{eq:reward-ept}). We simply refer to \armab as in the online setting in the rest of this paper.

Although the above definition is similar to that in stochastic settings, the fundamental difference is that the offline policy $\pi^{opt}$ is only optimal when it is defined over all $T$ episodes, and it is not guaranteed to be optimal in each episode, which is the case for stochastic setting. This is because the adversarial rewards can change arbitrarily between episodes rather than following some fixed (unknown) distribution as in the stochastic setting. This fundamental difference will necessitate new techniques in terms of regret characterization, which we will discuss in details in Section~\ref{sec:analysis}.

\section{RL Algorithm for \armab}\label{sec:learning}

We show that it is possible to design a RL algorithm to solve the regret minimization problem \eqref{eq:regret} for the computationally intractable \armab. Specifically, we leverage the popular UCRL-based algorithm to the online adversarial RMAB setting, and
develop an episodic RL algorithm named \ucarmab.  There are four key components of our algorithm: (1) maintaining a confidence set of the transition functions; (2) using online mirror descent (OMD) to solve a relaxed version of \armab in terms of occupancy measure to deal with adversarial rewards; (3) constructing an adversarial reward estimator to deal with bandit feedback;  and (4) designing a low-complexity index policy to ensure that the instantaneous activation constraint is satisfied in each decision epoch.  We summarize our \ucarmab in Algorithm \ref{alg:UCB}, which operates in an episodic manner with a total of $T$ episodes and each episode including $H$ decision epochs. For simplicity, let $\tau_t:= H(t-1)+1$ be the starting time of the $t$-th episode.

\begin{algorithm}[t]
	\caption{\ucarmab}
	\label{alg:UCB}
	\begin{algorithmic}[1]
     \REQUIRE  Initialize $C_n^{1}(s,a)=0,$  $\hat{P}_n^{1}(s^\prime|s,a)=1/|\cS|$	
		\FOR{$t=1,2,\cdots,T$}
	 \STATE Construct $\cP_n^{t}(s,a)$   according to \eqref{eq:confidence_ball} at $\tau_t$; 		
  \STATE Construct the adversarial reward estimator $\hat{r}_n^t(s,a), \forall s, a, n$ according to \eqref{eq:reward_estimator};
  \STATE Obtain a relaxed \armab \eqref{eq:relaxed_constraint} in terms of occupancy measure, and solve \eqref{eq:UCB_extended} with OMD; 
  \STATE Construct an index policy $\pi^{t}$ according to~(\ref{eq:importance-index}). 
		\ENDFOR
	\end{algorithmic}
\end{algorithm}

\subsection{Confidence Sets}\label{sec:confidence-set}

As discussed in Section~\ref{sec:model}, \armab has two components: a stochastic transition function, and an adversarial reward function for each arm. Since transition functions are unknown to the DM, we maintain confidence sets via past sample trajectories, which contain true transition functions $P_n, \forall n$ with high probability. Specifically, \ucarmab maintains two counts for each arm $n$.  Let $C_n^{t-1}(s,a), \forall n\in[N]$ be the number of visits to state-action pairs $(s,a)$ until $\tau_t$, and $C_n^{t-1}(s,a, s^\prime), \forall n\in[N]$ be the number of transitions from $s$ to $s^\prime$ under action $a$. At episode $t$, \ucarmab updates these two counts as: $\forall (s,a,s^\prime)\in\mathcal{S}\times\mathcal{A}\times\mathcal{S}$
\begin{align*}
    C_n^{t}(s,a)&=C_n^{t-1}(s,a)+\sum_{h=1}^{H}\mathds{1}(S_n^{t,h}=s,A_n^{t,h}=a),\displaybreak[0]\\
    C_n^{t}(s,a, s^\prime)&=C_n^{t-1}(s,a, s^\prime)\displaybreak[1]\\
    &+\sum_{h=1}^{H}\mathds{1}(S_n^{t,h+1}=s^\prime|S_n^{t,h}=s,A_n^{t,h}=a).
\end{align*}
\ucarmab estimates the true transition function by the corresponding empirical average as:
\begin{align}\label{eq:empirical_est_P}
\hat{P}_n^{t}(s^\prime|s,a)=\frac{C_n^{t-1}(s,a,s^\prime)}{\max\{C_n^{t-1}(s,a),1\}},
\end{align}
and then defines confidence sets at episode $t$ as 
\begin{align}\label{eq:confidence_ball} 
 \cP_n^{t}(s,a)&:=\{\tilde{P}_n^t(s^\prime|s,a),\forall s^\prime:\nonumber\\
 &\quad\quad|\tilde{P}_n^t(s^\prime|s,\!a)\!-\!\hat{P}_n^{t}(s^\prime|s,\!a)|\!\leq\! \delta_n^{t}(s,\!a)\},
\end{align}
where the confidence width $\delta_n^{t}(s,a)$ is built according to the Hoeffding inequality \citep{maurer2009empirical} as: for $\epsilon\in(0,1)$ 
\begin{align}
\delta_n^{t}(s,a)\!=\!\sqrt{\frac{1}{2C_n^{t-1}(s,a)}\!\log\!\Big(\frac{4|\cS||\cA|N(t\!-\!1)H}{\epsilon}\Big)}.
\end{align}

\begin{lemma}
With probability at least $1-2\epsilon$, the true transition functions are within the confidence sets, i.e., $P_n\in \cP_n^{t}$, $\forall n\in[N], t\in[T].$
\end{lemma}

\subsection{Solving Relaxed \armab with OMD}\label{sec:OMD}
Recall that solving \armab is computationally expensive even in the offline setting \citep{papadimitriou1994complexity}. To tackle this challenge, we first relax the instantaneous activation constraint, i.e., the activation constraint is satisfied on average, and obtain the following relaxed problem of 
\begin{align}\label{eq:relaxed_constraint}
\max_{\pi}~ &R(T, \pi):=\sum_{t=1}^T  R_t(\pi)\displaybreak[0]\nonumber\\ 
   \text{s.t.}~ & \mathbb{E}_{\pi}\left[\sum_{n=1}^N A_n^{t,h}\right]\leq B, ~h\in[H], t\in[T]. 
\end{align}
It turns out that this relaxed \armab can be equivalently transformed into a linear programming (LP) using occupancy measure \citep{altman1999constrained}. 
More specifically, the occupancy measure $\mu$ of a policy $\pi$ for a finite-horizon MDP is defined as the expected number of visits to a state-action pair $(s, a)$ at each decision epoch $h$.  Formally,
\begin{align}\label{eq:OM}
\mu^\pi=& \Big\{\mu_n(s,a;h)=\mathbb{P}(S_n^h=s, A_n^h=a)\nonumber\\
&\qquad: \forall n\in[N], s\in\cS, a\in\cA, h\in[H] \Big\}.
\end{align}
It can be easily checked that occupancy measures satisfy the following two properties. First, 
\begin{align}
\sum_{s\in\cS}\sum_{a\in\cA}\mu_n(s,a;h)=1, \forall n\in[N], h\in[H],
\end{align}
with $0\leq \mu_n(s,a;h)\leq 1$.  Hence the occupancy measure $\mu_n$, $\forall n$ is a probability measure. Second, the fluid balance exists in occupancy measure transitions as
\begin{align}
\sum_{a\in\mathcal{A}}\!\!\mu_n(s,a;h)\!=\!\!\sum_{s^\prime}\!\sum_{ a^\prime}\!\mu_n(s^\prime,\!a^\prime; h\!-\!1)P_n(s^\prime,\!a^\prime,\!s).
\end{align}
For ease of presentation, we relegate the details of the equivalent LP of~(\ref{eq:relaxed_constraint}) to the supplementary materials.

\begin{remark}\label{remark:OM}
Occupancy measure has been widely used in adversarial MDP \cite{rosenberg2019online, jin2020learning} and CMDP \cite{qiu2020upper}. Since there exists a stationary policy in these settings, the regret minimization problem in \citet{rosenberg2019online, jin2020learning, qiu2020upper} can be equivalently reduced to an online linear optimization in terms of the stationary occupancy measure. Unlike these works, there is no such stationary policy for our considered finite-horizon RMAB with the instantaneous activation constraint, and hence we cannot reduce our regret~(\ref{eq:regret}) into a linear optimization using stationary occupancy measure. To address this additional challenge, we leverage the time-dependent occupancy measure~(\ref{eq:OM}), and the regret minimization calls for different proof techniques (see Section~\ref{sec:analysis} for details). 
\end{remark}
 
Unfortunately, we cannot solve this LP since we have no knowledge about the true transition functions and adversarial rewards. Similar to the stochastic setting \cite{xiong2022AAAI, xiong2022Neurips} and with the confidence sets defined in Section~\ref{sec:confidence-set}, we can further rewrite this LP as an extended LP by leveraging the \textit{state-action-state occupancy measure} $z_n^t(s, a, s^\prime; h)$ defined as $z_n^t(s, a, s^\prime; h)=P_n(s^\prime|s,a)\mu_n^t(s,a;h)$ to express the confidence intervals of the transition probabilities. Unlike \citet{xiong2022AAAI, xiong2022Neurips}, the adversarial rewards can change arbitrarily between episodes, and hence we also need to guarantee that the updated occupancy measure in episode $t$ does not deviate too much away from the previously chosen occupancy measure in episode $t-1$. Thus, we further incorporate $D(z^t||z^{t-1})$ into the objective function, which is the unnormalized Kullback-Leible (KL) divergence between two occupancy measures, which is defined as
\begin{align}\label{eq:KL_divergence}
\hspace{-0.3cm}D(z^t||z^{t-1})&:=\sum_{h=1}^{H}\sum_{s,a,s^\prime}z^t(s,a,s^\prime;h)\ln\frac{z^t(s,a,s^\prime;h)}{z^{t-1}(s,a,s^\prime;h)}\nonumber\\
&\qquad-z^t(s,a,s^\prime;h)+z^{t-1}(s,a,s^\prime;h).
\end{align}
The DM needs to solve a non-linear problem over $z^t:=\{z_n^t(s, a, s^\prime; h), \forall n\in[N]\}$ for a given parameter $\eta>0$:  
\begin{align}\label{eq:UCB_extended}
\max_{z^t}&\sum_{h=1}^{H}\!\sum_{n=1}^{N}\!\sum_{(s,a,s^\prime)}\!\!\! \eta z_n^t(s,a, s^\prime;h)\hat{r}_n^{t-1}(s,a)\!-\!D(z^t||z^{t-1})\nonumber\allowdisplaybreaks\\
\mbox{ s.t.} &\sum_{n=1}^{N}\sum_{(s,a,s^\prime)} az_n^t(s,a,s^\prime; h) \le B ,\quad~\forall h\in[H], \nonumber\allowdisplaybreaks\\
&{\sum_{a, s^\prime}} z_n^t(s,a,s^\prime;h)\!=\!\!\!\sum_{s^\prime, a^\prime}\!z_n^t(s^\prime, a^\prime, s; h-1),\forall h\!\in\![H],  \nonumber\allowdisplaybreaks\\ 
&~\frac{z_n^t(s,a,s^\prime;h)}{\sum_y z_n^t(s,a,y;h)}-(\hat{P}_n^t(s^\prime|s,a)+\delta_n^t(s,a))\leq 0, \nonumber\allowdisplaybreaks\\ 
&-\frac{z_n^t(s,a,s^\prime;h)}{\sum_y z_n^t(s,a,y;h)}\!+\!(\hat{P}_n^t(s^\prime|s,a)\!-\!\delta_n^t(s,a))\!\leq\! 0, \nonumber\allowdisplaybreaks\\ 
& z_n^t(s,a, s^\prime;h)\!\geq\! 0,  \forall s,s^\prime\!\in\mathcal{S}, a\!\in\!\mathcal{A},\forall n\!\in\![N], 
\end{align}
where $\hat{r}_n^{t-1}(s,a)$ is the estimated adversarial reward due to bandit feedback, and we formally define it in Section~\ref{sec:estimator}. 

This problem has $O(|\mathcal{S}|^2|\mathcal{A}|HN)$ constraints and decision variables. Inspired by \citet{rosenberg2019online}, we solve~(\ref{eq:UCB_extended}) via OMD to choose the occupancy measure for each episode $t$. We first solve a unconstrained problem by setting $\tilde{z}^{t}(s,a,s^\prime; h)=z^{t-1}(s,a,s^\prime;h)e^{\eta \hat{r}^{t-1}(s,a)}$. Then we project the unconstrained maximizer $\tilde{z}^{t}$ into the feasible set $\cZ^t$, which is defined by the constraints in \eqref{eq:UCB_extended}. This can be reduced to a convex optimization problem, and can be efficiently solved using iterative methods \citep{boyd2004convex}. For ease of readability, we relegate the details of solving \eqref{eq:UCB_extended} to the supplementary materials.  Denote the optimal solution to~(\ref{eq:UCB_extended}) as $z^{t,\star}$.

\subsection{Adversarial Reward Estimators}\label{sec:estimator} 
Since we consider a challenging bandit feedback setting, where the adversarial rewards in each decision epoch are revealed to the DM only when the state-action pairs are visited, we need to construct adversarial reward estimators based on observations. Specifically, we build upon the inverse importance-weighted estimators based on the observation of counts for each state-action pairs. Given the trajectory for episode $t$, a straightforward estimator is 
\begin{align}\label{eq:unbias}
\frac{r_n(s,a)}{\max\{c_n^t(s,a),1\}/H}\mathds{1}(\exists h, S_n^{t,h}=s, A_n^{t,h}=a), 
\end{align}
where $c_n^t(s,a):=\sum_{h=1}^{H}\mathds{1}(S_n^{t,h}=s,A_n^{t,h}=a)$ is the number of visits to state-action pair $(s,a)$ in episode $t$. For simplicity, we denote $\bar{r}_n^{t}(s,a):=\frac{r_n(s,a)}{\max\{c_n^t(s,a),1\}/H}$. Clearly, $\bar{r}_n^{t}(s,a)\mathds{1}(\exists h, S_n^{t,h}=s, A_n^{t,h}=a)$ is an unbiased estimator of ${r}_n^{t}(s,a)$ from the above definition. A key difference between the unbiased estimator in \eqref{eq:unbias} and those in previous works on adversarial MDPs  \citep{jin2020learning, rosenberg2019online} lies in the construction of the denominator in \eqref{eq:unbias}. Specifically, \citet{jin2020learning, rosenberg2019online} considered MDPs with a underlying stationary policy, and hence the evolution of the dynamics in the considered MDPs will converge to this stationary policy. This enables the construction of the denominator in the estimator using the occupancy measure for each for each state-action pair under such a policy. In contrast, it is known that the dynamics in RMAB cannot converge to the stationary policy due to the fact that RMABs implement an index policy (see Section~\ref{sec:index}) to deal with the instantaneous activation constraint, and it is only provably in the asymptotic regime \cite{weber1990index,verloop2016asymptotically}. This render the approaches in  \citet{jin2020learning, rosenberg2019online} not applicable to ours, and necessitates different methods to construct the estimator as in \eqref{eq:unbias}.

Since we consider the bandit feedback, we further leverage the idea of implicit exploration as inspired by \citet{neu2015explore,jin2020learning} to further encourage exploration. Specifically, we further increase $\bar{r}_n^{t}(s,a)$ with a bonus term $\delta_n^t(s,a)$ to obtain a biased estimator $(\bar{r}_n^{t}(s,a)+\delta_n^t(s,a))\mathds{1}(\exists h, S_n^{t,h}=s, A_n^{t,h}=a)$. Since $r_n^t(s,a)\in[0,1]$ and to guarantee that this biased estimator is an overestimate, we further add the term $1-\mathds{1}(\exists h, S_n^{t,h}=s, A_n^{t,h}=a)$. Thus, our final adversarial reward estimator is $\hat{r}_n^{t}(s,a)$
\begin{align}\label{eq:reward_estimator}
    =&\min\Big((\bar{r}_n^{t}(s,a)+\delta_n^t(s,a)) \mathds{1}(\exists h, S_n^{t,h}=s, A_n^{t,h}=a), 1\Big) \nonumber\\
    &\qquad\qquad+ 1-\mathds{1}(\exists h, S_n^{t,h}=s, A_n^{t,h}=a).
\end{align}
Given the above constructions, it is clear that $\hat{r}_n^{t}(s,a)$ is a biased estimator upper-bounded by $1$ and is overestimating  $r_n^t(s,a), \forall s, a, n, t$ with high probability. Using overestimates for adversarial learning with bandit feedback can be viewed as an optimism principle to encourage exploration. This is beneficial for the regret characterization as the deterministic overestimation as in \citet{jin2020learning} to guarantee a tighter regret bound.

\begin{remark}
The reward estimator in stochastic RMAB \citep{xiong2022AAAI, xiong2022Neurips} is simply defined as the sample mean using all trajectories up to episode $t$, which cannot be applied to our adversarial RMAB. This is due to the fact that the rewards are assumed to be drawn from an unknown but fixed distribution across all episodes for stochastic RMAB, while rewards can change arbitrarily between episodes for adversarial RMAB. The bandit-feedback setting further differentiates our adversarial RMAB from classical stochastic ones. Finally, we note that \citet{jin2020learning} considered the bandit feedback for adversarial MDP, and also constructed an adversarial loss/reward estimator. Since the regret minimization problem in \citet{jin2020learning} can be reduced to an online linear optimization using stationary occupancy measure (see Remark~\ref{remark:OM}), the estimator can be constructed directly using the stationary occupancy measure.  Since there is no such stationary policy for our finite-horizon adversarial RMAB, this makes the estimator in \citet{jin2020learning} not directly applicable to ours, and necessitates different construction techniques as discussed above. 
\end{remark}

\begin{figure*}
    \centering
    \includegraphics[width=0.95\textwidth]{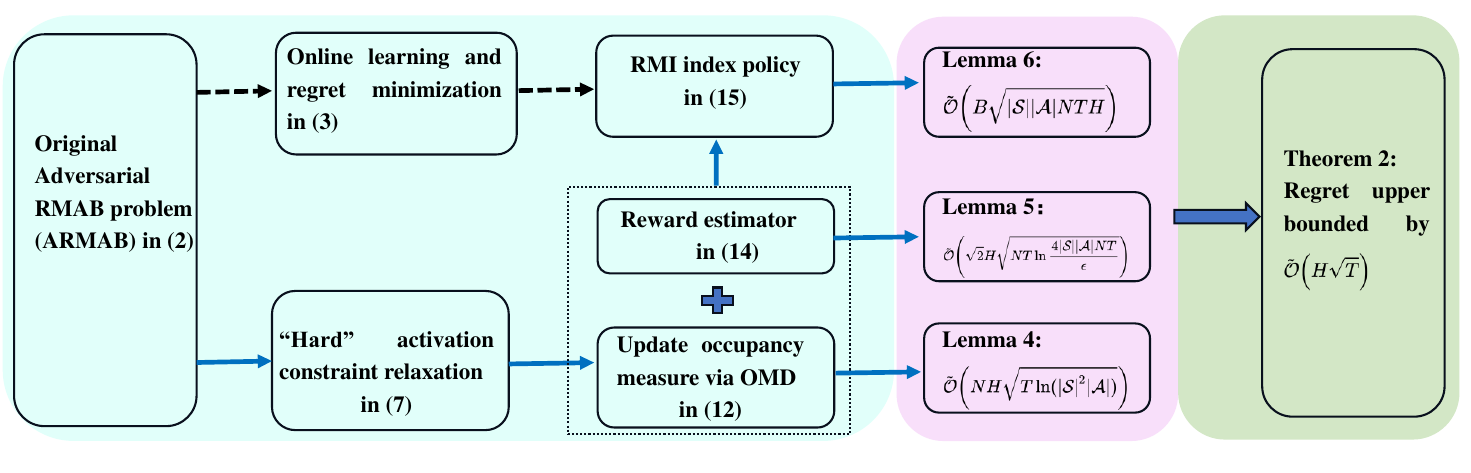}
    \caption{The workflow of \ucarmab and its regret analysis. The dashed arrows present the aimed procedures for solving the original problem in \eqref{eq:orginal_P}, and the solid arrows show the true procedures of \ucarmab. By relaxing the ``hard'' activation constraint as shown in \eqref{eq:relaxed_constraint},  \ucarmab updates occupancy measure via OMD as in \eqref{eq:UCB_extended} (see Section~\ref{sec:OMD}), combined with the adversarial reward estimator in \eqref{eq:reward_estimator} (see Section~\ref{sec:estimator}). Then, it establishes the RMI index policy in \eqref{eq:importance-index} (see Section~\ref{sec:index}). These correspond to the \textbf{three sources of learning regret}, i.e., regret due to (i) OMD online optimization (Lemma \ref{lem:term1}), (ii) bandit-feedback adversarial reward (Lemma \ref{lem:term2}), and (iii) the RMI index policy (Lemma \ref{lem:term3}).}

    \label{fig:Flowchart}
\end{figure*}

\subsection{Index Policy for \armab}\label{sec:index}
Unfortunately, the optimal solution $z^{t,\star}$ to~(\ref{eq:UCB_extended}) is not always feasible for our \armab due to the fact that the instantaneous activation constraint in \armab must be satisfied in each decision epoch rather than on the average sense as in~(\ref{eq:UCB_extended}). Inspired by \citet{xiong2022AAAI,xiong2022Neurips}, we further construct an index policy on top of $z^{t,\star}$ that is feasible for \armab. Specifically, since $\cA:=\{0,1\}$, i.e., an arm can be either active or passive at each decision epoch $h$, we define the index assigned to arm $n$ in state $S_n^{t, h}=s$ at decision epoch $h$ of episode $t$ to be as 
\begin{align}\label{eq:importance-index}
\mathcal{I}_n^t(s;h):=\frac{\sum_{s^\prime}z_n^{t,\star}(s,1,s^\prime;h)}{\sum_{b,s^\prime}z_n^{t,\star}(s,b,s^\prime;h)},\quad \forall n\in[N].
\end{align}
We call this the reward-maximizing index (RMI) since $\mathcal{I}_n^t(s;h)$ indicates the probability of activating arm $n$ in state $s$ at decision epoch $h$ of episode $t$ towards maximizing the total expected adversarial rewards. To this end, we rank all arms according to their indices in a non-increasing order, and activate the set of $B$ highest indexed arms at each decision epoch $h$. All remaining arms are kept passive at decision epoch $h$. We denote the resultant RMI policy as $\pi^{t}:=\{\pi_{n}^{t}, \forall n\}$, and execute it in episode $t$.

\begin{theorem}\label{thm:asympt_opt}
The RMI policy is asymptotically optimal when both the number of arms and instantaneous activation constraint are enlarged by $\rho$ times with $\rho\rightarrow\infty$.
\end{theorem}

\section{Analysis}\label{sec:analysis}
In this section, we bound the regret of our \ucarmab.

\subsection{Main Results}

\begin{theorem}\label{thm:regret}
With probability at least $1-3\epsilon$, the regret of \ucarmab with $\eta=\sqrt{\frac{\ln(|\cS|^2|\cA|)}{T}}$ satisfies 
\begin{align}\label{eq:regret_thm}
&\Delta(T)={\tilde{\mathcal{O}}}\Bigg(NH\sqrt{T\ln(|\cS|^2|\cA|)}\nonumber\\
&+H\sqrt{2NT\ln\frac{4|\cS||\cA|NT}{\epsilon}}\!+\!B\sqrt{|\cS||\cA|NTH}\Bigg).
\end{align}
\end{theorem}

The regret in \eqref{eq:regret_thm} contains three terms. The first term is the regret due to the OMD online optimization for occupancy measure updates (Section~\ref{sec:OMD}). The second term represents the regret due to bandit-feedback of adversarial rewards (Section~\ref{sec:estimator}). The third term comes from the implementation of our RMI policy for \armab (to satisfy the instantaneous activation constraint,  Section~\ref{sec:index}). Clearly, the regret of \ucarmab is in the order of $\tilde{\mathcal{O}}(H\sqrt{T})$. This is the same as that for adversarial MDP \citep{rosenberg2019online, jin2020learning} and CMDP \citep{qiu2020upper}. However, none of them consider an instantaneous activation constraint as in our \armab, which requires us to design a low-complexity index policy, and explicitly characterize its impact on the regret.  Although our regret bound exhibits a gap (i.e., $\sqrt{H}$ times larger) to that of stochastic RMAB \citep{xiong2022AAAI, xiong2022Neurips}, to the best of our knowledge, our result is the first to achieve $\tilde{\mathcal{O}}(\sqrt{T})$ regret. Recall that comparing to the stochastic RMAB, we are considering a harder problem with a challenging setting, i.e., the rewards can change arbitrarily between episodes rather than following a fixed distribution, and with bandit feedback where only the adversarial reward of visited state-action pairs are revealed to the DM. This challenging setting thus requires us to design a novel adversarial reward estimator coupled with the OMD online optimization procedure.

\subsection{Proof Sketch}
As discussed earlier, \armab~(\ref{eq:orginal_P}) is computationally intractable, and hence we cannot directly solve it and transform the regret minimization in~(\ref{eq:regret}) into an online linear optimization problem. This makes existing regret analysis for adversarial MDP  \citep{rosenberg2019online, jin2020learning} and CMDP \citep{qiu2020upper} not directly applicable to ours, and necessitates different proof techniques.  To address this challenge and inspired by stochastic RMAB \citep{xiong2022AAAI,xiong2022Neurips}, we instead work on the relaxed problem~(\ref{eq:relaxed_constraint}), which achieves a provably upper bound on the adversarial rewards of \armab~(\ref{eq:orginal_P}).  In other words, the occupancy measure-based solutions to~(\ref{eq:relaxed_constraint}) provide an upper bound of the optimal adversarial reward $R( T, \pi^{opt})$ achieved by the offline optimal policy $\pi^{opt}$ of \armab \eqref{eq:orginal_P}. We state this result formally in the following lemma. 

\begin{lemma}\label{lem:upperbound}
There exists a set of occupancy measures $\mu_{\pi}^*:=\{\mu^*_n(s,a;h), \forall n\in[N], s\in\cS, a\in\cA, h\in[h]\}$ under policy $\pi^*$  that optimally solve the equivalent LP of the relaxed problem~(\ref{eq:relaxed_constraint}). In addition, $\sum_{t=1}^T \langle{\mu_\pi^*}, r^t\rangle$ is no less than $R( T, \pi^{opt})$, where $r^t:=\{r_n^t(s,a), \forall n, s, a\}$. 
\end{lemma}

The proof of Theorem~\ref{thm:regret} then starts with a regret decomposition. Unlike stochastic RMAB \citep{xiong2022AAAI,xiong2022Neurips}, we cannot simply decompose the regret over episodes. This is because in our \armab, the adversarial rewards can change arbitrarily between episodes, and the offline optimal policy $\pi^{opt}$ is only optimal when it is defined over all $T$ episodes and it is not guaranteed to be optimal in each episode (see Section~\ref{sec:model}). As a result, the regret analysis for stochastic RMAB \citep{xiong2022AAAI,xiong2022Neurips} is not applicable to ours. To address this challenge, we instead decompose the regret in terms of its coming sources.  Specifically, we first need to characterize the regret due to solving the relaxed problem with OMD in terms of occupancy measure (Section~\ref{sec:OMD}). We then bound the regret due to the biased overestimated reward estimator (Section~\ref{sec:estimator}). Finally, we bound the regret of implementing our RMI policy (Section~\ref{sec:index}). We visualize these three steps in Figure \ref{fig:Flowchart} and provide a proof sketch herein. Combining them together gives rise to our main theoretical results in Theorem~\ref{thm:regret}.

\textbf{Regret decomposition.} We formally state our regret decomposition in the following lemma.  
\begin{lemma}\label{lemma:regret_decomp}
{Denote $\pi^{*}$ as the optimal policy  to equivalent LP of the relaxed problem~(\ref{eq:relaxed_constraint}) and $\{\tilde{\pi}^t, \forall t\}$ are polices executed on the MDPs at each episode with transition kernels selected from the confidence set defined in \eqref{eq:confidence_ball}.}
Let $R(T, \pi^{*}, \{\hat{r}_n^t, \forall n, t\})$ be the total adversarial rewards achieved by policy $\pi^*$ with overestimated rewards $\{\hat{r}_n^t, \forall n, t\}$. Denote 
$R(T, \{\tilde{\pi}^t, \forall t\}, \{\hat{r}_n^t, \forall n, t\})$ and  $R( T, \{\tilde{\pi}^t, \forall t\}, \{r_n^t, \forall n, t\})$ as the total adversarial rewards achieved by policy $\{\tilde{\pi}^t, \forall t\}$ with overestimated rewards $\{\hat{r}_n^t, \forall n, t\}$ and true reward $\{{r}_n^t, \forall n, t\}$, respectively.
The regret in \eqref{eq:regret} can be upper bounded as
\begin{align}\label{eq:lemma2}
    &\Delta(T)\!\leq {\underset{Term_0\leq 0}{\underbrace{R(T, \pi^{opt})-R( T, \pi^{*}, \{\hat{r}_n^t, \forall n, t\}), \forall n, t\})}}}\nonumber\allowdisplaybreaks\\
    &+\underset{Term_1}{\underbrace{R( T, \pi^{*}, \{\hat{r}_n^t, \forall n, t\})\!-\!R( T, \{\tilde{\pi}^t, \forall t\}, \{\hat{r}_n^t, \forall n, t\})}}\nonumber\allowdisplaybreaks\\
    &+\!\underset{Term_2}{\underbrace{R( T, \{\tilde{\pi}^t, \forall t\}, \{\hat{r}_n^t, \forall n,t\})\!-\!R( T, \{\tilde{\pi}^t, \forall t\}, \{{r}_n^t, \forall n,t\})}}\nonumber\allowdisplaybreaks\\
    \vspace{-0.2cm}
    &+\underset{Term_3}{\underbrace{R( T, \{\tilde{\pi}^t, \forall t\}, \{{r}_n^t, \forall n,t\})-\sum_{t=1}^T R_t(\pi^t)}}. \allowdisplaybreaks
\end{align}
\vspace{-0.5cm}
\end{lemma}

 Specifically, $Term_0\leq 0$ holds  due to Lemma \ref{lem:upperbound} and the overestimation of adversarial rewards.  $Term_1$ is the performance gap between the optimal policy $\pi^*$ of the relaxed problem~\eqref{eq:relaxed_constraint} and the OMD updated policy $\{\tilde{\pi}^{t}, \forall t\}$ under the plausible MDP selected from the confidence set in \eqref{eq:confidence_ball}. Since we leverage the OMD to update the occupancy measures, to bound $Term_1$, the key is to connect items in $Term_1$ with the occupancy measure, which leverages the result in Lemma \ref{lem:upperbound}. $Term_2$ is the regret due to the overestimation of the adversarial reward estimator, i.e., $\hat{r}_n^t\geq r_n^t, \forall n, t$.  $Term_3$ is the performance gap between the policy $\{\tilde{\pi}^t, \forall t\}$ in the optimistic plausible MDP and the learned RMI index policy $\{\pi^t, \forall t\}$ for the true MDP. We bound it based on the count of visits of each state-action pair. 

\textbf{Bounding $Term_1$.} We first bound $Term_1$, i.e., the regret due to OMD online optimization.

\begin{lemma}\label{lem:term1}
With probability $1-2\epsilon$, 
we have $Term_1\leq \frac{NH\ln(|\cS|^2|\cA|)}{\eta}+\eta{NH(T+1)}.$
\end{lemma}

Bounding $Term_1$ is equivalent to bound the inner product $\sum_{t=1}^T\sum_{n=1}^N\langle \mu_n^* - z_n^t, \hat{r}_n^t \rangle$, with $z_n^t, \forall n \in[N]$ being controlled by the OMD updates, a proper $\eta$ (as given in Theorem~\ref{thm:regret}) is required to guarantee $\tilde{\mathcal{O}}(H\sqrt{T})$ regret.

\textbf{Bounding $Term_2$.} We then bound $Term_2$, i.e., the regret due to bandit-feedback adversarial reward.

\begin{lemma}\label{lem:term2}
With probability $1-3\epsilon$, we have $Term_2\leq H\sqrt{2NT\ln\frac{4|\cS||\cA|NT}{\epsilon}}+HN\sqrt{NT\ln \frac{1}{\epsilon}}.$
\end{lemma}\vspace{-0.2cm}
Bounding $Term_2$ requires us to bound the inner product
$\sum_{t=1}^T\sum_{n=1}^N\langle  z_n^t, \hat{r}_n^t-r_n^t \rangle$, which is dominated by the gap of $\hat{r}_n^t-r_n^t$. An overestimation can guarantee that $\langle  z_n^t, \hat{r}_n^t-r_n^t \rangle\geq 0, \forall n\in[N], t\in[T].$

\textbf{Bounding $Term_3$.} Finally, we bound $Term_3$, i.e., the regret due to RMI index policy. 
 
\begin{lemma}\label{lem:term3}
With probability $1-2\epsilon$,  we have $Term_3\leq\Big(\sqrt{2\ln\frac{4|\cS||\cA|NTH}{\epsilon}}+2B\Big)\sqrt{|\cS||\cA|NTH}.$
\end{lemma}\vspace{-0.2cm}
Since the adversarial rewards (i.e., $\hat{r}_n^t$) do not impact $Term_3$, we decompose  $Term_3$ into each episode.  The key  is to characterize the number of visits of each state-action pair, which is related to the instantaneous activation constraint $B$ and has a $\tilde{\mathcal{O}}(B\sqrt{TH})$ regret.

\begin{figure}[t]
\centering
\begin{minipage}{.24\textwidth}
\centering
\includegraphics[width=1\columnwidth]{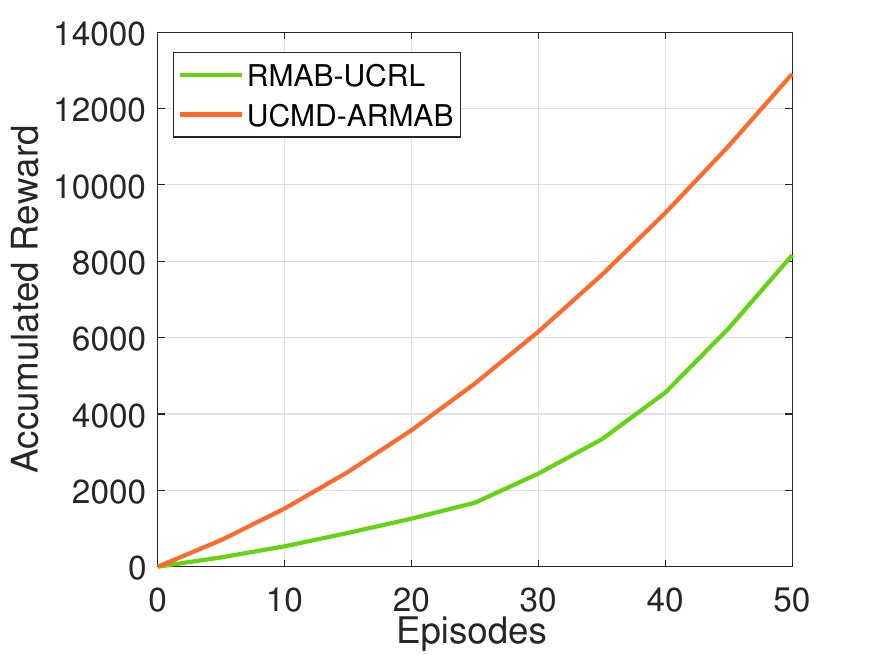}
\subcaption{Accumulated Reward.}
 \label{fig:reward}
\end{minipage}\hfill
\begin{minipage}{.24\textwidth}
\centering
\includegraphics[width=1\columnwidth]{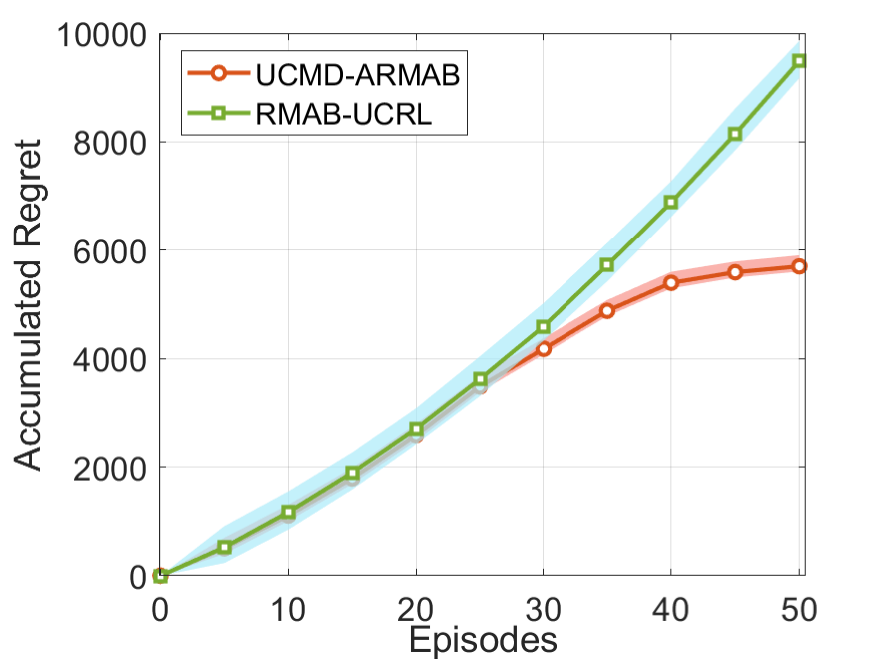}
\subcaption{Accumulated Regret.}
 \label{fig:regret}
\end{minipage}
\caption{The learning performance comparison for case study-CPAP.}
        \label{fig:examples}
\end{figure}

\begin{figure}[t]
\centering
\begin{minipage}{.24\textwidth}
\centering
\includegraphics[width=1\columnwidth]{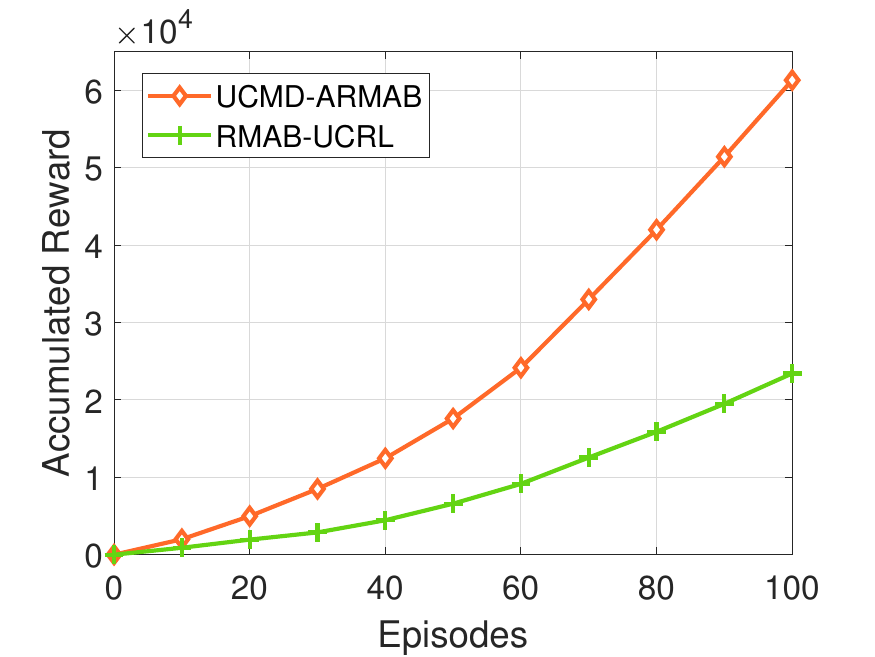}
\subcaption{Accumulated reward.}
 \label{fig:reward2}
\end{minipage}\hfill
\begin{minipage}{.24\textwidth}
\centering
\includegraphics[width=1\columnwidth]{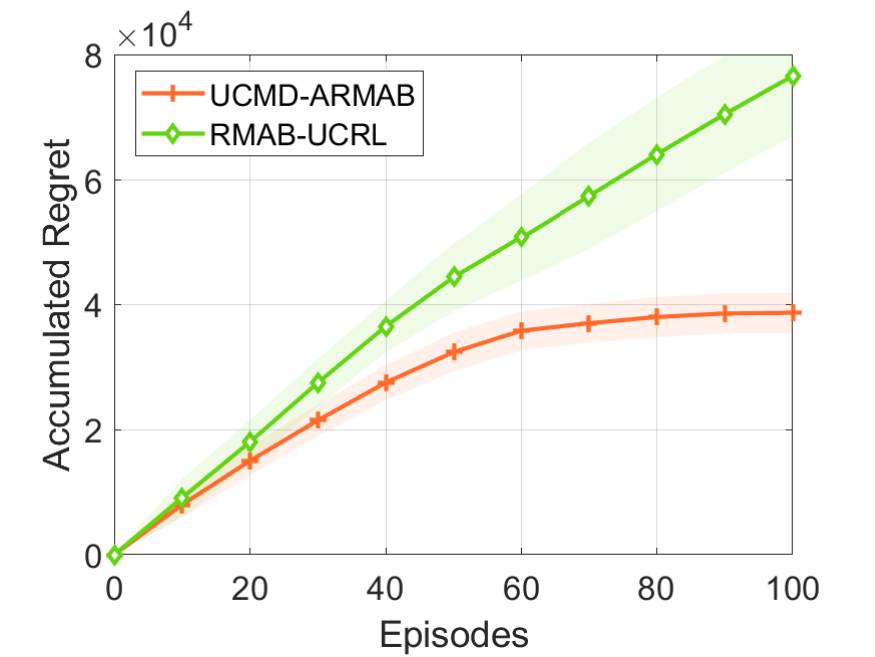}
\subcaption{Accumulated regret.}
 \label{fig:regret2}
\end{minipage}
\caption{The learning performance comparison for case study-A Deadline Scheduling Problem.}
        \label{fig:examples2}
\end{figure}

\section{Numerical Study}
In this section, we demonstrate the utility of \ucarmab by evaluating it under two real-world applications of RMAB in the presence of adversarial rewards, i.e., the continuous positive airway pressure therapy (CPAP)  \citep{kang2013markov,herlihy2023planning,li2022towards, wang2024online} and a deadline scheduling problem \cite{xiong2022AAAI}. The detailed setup of these two problems are provided in Appendix~\ref{sec:experiments}.

We compare our \ucarmab with a benchmark named RMAB-UCRL \cite{xiong2022Neurips}, which was developed for stochastic RMAB, in terms of accumulated rewards and accumulated regret. As observed from Figure~\ref{fig:examples} and Figure~\ref{fig:examples2}, our \ucarmab significantly outperforms RMAB-UCRL in adversarial settings. In particular, our \ucarmab achieves a significant improvement over the accumulated reward as shown in in Figure~\ref{fig:reward} and Figure~\ref{fig:reward2},
and exhibits a provably sublinear regret as shown in Figure~\ref{fig:regret} and Figure~\ref{fig:regret2}, while the regret of RMAB-UCRL increases exponentially under adversarial settings.
These verify the effectiveness of the proposed \ucarmab on handling adversarial RMABs.

\section*{Acknowledgements} 

This work was supported in part by the National Science Foundation (NSF) grants 2148309 and 2315614, and was supported in part by funds from OUSD R\&E, NIST, and industry partners as specified in the Resilient \& Intelligent NextG Systems (RINGS) program. This work was also supported in part by the U.S. Army Research Office (ARO) grant W911NF-23-1-0072, and the U.S. Department of Energy (DOE) grant DE-EE0009341. Any opinions, findings, and conclusions or recommendations expressed in this material are those of the authors and do not necessarily reflect the views of the funding agencies.

\section*{Impact Statements}

This paper presents work whose goal is to advance the field of Machine Learning. There are many potential societal consequences of our work, none which we feel must be specifically highlighted here.

\nocite{langley00}


\bibliography{refs,refs2}
\bibliographystyle{icml2024}

\newpage
\appendix
\onecolumn

\section{Details and Proofs in Section \ref{sec:learning}}
We provide all omitted details and proofs of the key lemmas of the main paper in this appendix.

\subsection{Efficient Solver of Updating Occupancy Measure \eqref{eq:UCB_extended}}
In this subsection, we provide details on how to efficiently solve the updating occupancy measure problem in \eqref{eq:UCB_extended}.
Similar to \citet{rosenberg2019online}, the solution of \eqref{eq:UCB_extended} can be yielded by decomposing the original problem into two subproblems. Specifically, the first subproblem is to solve the objective in \eqref{eq:UCB_extended} as an unconstrainted problem as 
\begin{align}
\tilde{z}^t = \arg\max_{z^t}&~\sum_{h=1}^{H}\sum_{n=1}^{N}\sum_{(s,a,s^\prime)} \eta z_n^t(s,a, s^\prime,h)\hat{r}_n^{t-1}(s,a) -D(z^t||z^{t-1}),  
\end{align}
which can be easily solved and has a closed-form solution as
\begin{align}\label{eq:solution_1st}
    \tilde{z}_n^{t}(s,a,s^\prime; h)=z_n^{t-1}(s,a,s^\prime;h)e^{\eta \hat{r}_n^{t-1}(s,a)}, \forall s\in\cS, a\in\cA, h\in[H], n\in[N].
\end{align}
The second subproblem is then to project the solution in \eqref{eq:solution_1st} to the feasible set of $z^t$ in episode $t$ by minimizing the KL-divergence of $z^t$ and $\tilde{z}^t$. To simplify the notations, we denote the feasible set of $z^t$ that satisfies the constraints in \eqref{eq:UCB_extended} as $\cZ^t$, and thus the subproblem is expresses as follows
\begin{align}\label{eq:solution2}
z^t = \arg\min_{z\in\cZ^t}&~D(z||\tilde{z}^t).  
\end{align}

Defining some Lagrangian parameters as
$\beta:\cS\mapsto \mathbb{R}$ and $\mu=(\mu^+, \mu^-)$ with $\mu^+, \mu^-:\cS\times\cA\times\cS\mapsto \mathbb{R}_{\geq 0}$, and $B_{n,t}^{\mu,\beta}(s,a,s^\prime,h)$ as
\begin{align*}
B_{n,t}^{\mu,\beta}(s,a,s^\prime,h)&= \mu^-(s,a,s^\prime,h)-\mu^+(s,a,s^\prime,h)+(\mu^+(s,a,s^\prime,h)+\mu^-(s,a,s^\prime,h))\delta_n^t(s,a)\\
&+\beta(s^\prime,h)-\beta(s,h)+\eta\hat{r}_n^t(s,a)+\sum_{s^{\prime\prime}}\hat{P}_n^t(s^{\prime\prime}|s,a)(\mu^+(s,a,s^{\prime\prime},h)-\mu^-(s,a,s^{\prime\prime},h)).
\end{align*}
Hence, the following theorem from \citet{rosenberg2019online} provides the solution of \eqref{eq:solution2} and the detailed proof can be found in \citet{jin2020learning}. 
\begin{theorem}[Theorem 4.2 \citep{rosenberg2019online}]
Let $t>1$ and define the function
\begin{align*}
    Z_n^{t,h}(\mu,\beta)=\sum_{s,a,s^\prime}z_n^t(s, a, s^\prime, h)e^{B_{n,t}^{\mu,\beta}(s, a, s^\prime, h)},
\end{align*}
the solution to \eqref{eq:solution2} is \begin{align*}
    z_n^{t+1}(s,a,s^\prime,h)=\frac{z_n^t(s,a,s^\prime,h)e^{B_{n,t}^{\mu^t,\beta^t}(s, a, s^\prime,h)}}{Z_n^{t,h}(\mu^t,\beta^t)},
\end{align*}
where $\mu^t,\beta^t$ are obtained by solving a convex optimization with non-negativity constraints as
\begin{align*}
    \mu^t,\beta^t=\arg\min_{\mu,\beta>0}\sum_{h=1}^H \ln Z_n^{t,h}(\mu,\beta).
\end{align*}

\end{theorem}

\subsection{Index Policy Design and proof of Theorem \ref{thm:asympt_opt}}\label{sec:indexpolicy}
Following the method of the celebrated Whittle index \cite{whittle1988restless} for conventional RMAB problems, we first relax
 the ``hard" activation constraint in \eqref{eq:orginal_P}  as an averaged constraint, which gives rise to the following relaxed problem. 
 \begin{align}
  \max_\pi \sum_{t=1}^T R_t~~ ~~\text{s.t.}~ \mathbb{E}_{\pi}\left[\sum_{n=1}^N A_n^{t,h}\right]\leq B ,h\in[H], t\in[T].
    \label{eq:relaxed_P}
\end{align}

Provided the definition of occupancy measure in \eqref{eq:OM}, the relaxed problem \eqref{eq:relaxed_P} can be reformulated as a linear programming (LP) \cite{altman1999constrained} expressed in the following lemma.

\begin{lemma}\label{prop:lp}
The relaxed problem \eqref{eq:relaxed_P} is equivalent to the following LP 
\begin{align}
    \max_{\mu}~ &\sum_{t=1}^T\sum_{n=1}^{N}\sum_{h=1}^{H}\sum\limits_{(s,a)} \mu_n(s,a;h)r_n^t(s,a) \label{eq:pmc_rel_lp}\displaybreak[0]\\
\text{s.t.}~& \sum_{n=1}^{N}\sum_{(s,a)} a\mu_n(s,a;h)\leq B,\quad~\forall h\in[H], \label{eq:pmc_rel_lp-1}\displaybreak[1]\\
&\sum_{a\in\mathcal{A}}\mu_n(s,a;h)\!=\!\sum_{(s^\prime, a^\prime)}\!\!\mu_n(s^\prime, a^\prime; h-1)P_n(s^\prime, a^\prime,s),\label{eq:pmc_rel_lp-2}\displaybreak[2]\\
& \mu_n(s,a;h)\geq 0, \quad\forall s\in\cS, a\in\mathcal{A}, h\in[H], n\in[N], \label{eq:pmc_rel_lp-3}
\end{align}
where~(\ref{eq:pmc_rel_lp-1}) is a restatement of the constraint in \eqref{eq:relaxed_P}; ~(\ref{eq:pmc_rel_lp-2}) indicates the transition of the occupancy measure from time slot $h-1$ to time slot $h$; and~(\ref{eq:pmc_rel_lp-3}) guarantees that the occupancy measures are non-negative. 
\end{lemma}

The RMI is designed upon the optimal occupancy measure
 $\mu^\star = \left\{\mu^\star_{n}(s,a;h): n\in[N], h\in[1,\ldots,H]\right\}$ to the above LP. The first step is to construct a Markovian randomized policy 
$
\xi^\star_n(s,a;h):=\frac{\mu^\star_n(s,a;h)}{\sum_{a\in \cA}\mu^\star_n(s,a; h)},
$ 
and then set $\xi^\star_n(s,1;h)$ as the RMI.   
Thereafter, we prioritize the users according to a decreasing order of their RMIs, and then activate the $B$ arms with the highest indices.  Note that our proposed RMI policy is well-defined even when the problem is not indexable~\cite{whittle1988restless, weber1990index}.

\textbf{Proof of Theorem \ref{thm:asympt_opt}.}  
We now show that our RMI policy is asymptotically optimal when both the number of arms $N$ and the activation constraint $B$ go to infinity while holding $B/N$ constant as that in Whittle \cite{whittle1988restless} and others \cite{weber1990index}.  %
For abuse of notation, let the number of arms be $\rho N$ and the resource constraint be $\rho W$ in the asymptotic regime with $\rho\rightarrow\infty.$  In other words, we consider $N$ different classes of arms with each class containing $\rho$ arms. 
Let $R^{\pi}(\rho B, \rho N)$ denote the expected total reward of the original problem~(\ref{eq:orginal_P}) under an arbitrary policy $\pi$ for such a system.
To show the asymptotical optimality of RMI, according to \citet{xiong2022AAAI}, it is sufficient to show that
\begin{align}\label{eq:Asym_opt}
\lim_{\eta\rightarrow \infty} \frac{1}{\rho}\Big(R^{\pi^{opt}}(\eta W, \eta N)-R^{\pi^\star}(\eta W, \eta N)\Big)=0,
\end{align}
with $\pi^*$ being the RMI index policy and $\pi^{opt}$ is the optimal one. 
The equation in
 \eqref{eq:Asym_opt} indicates that as the number of per-class arms goes to infinity, the average gap between the performance achieved by our RMI index policy $ \pi^\star$ and the optimal policy $\pi^{opt}$ of the original problem in \eqref{eq:orginal_P} tends to be zero.

The proof comes from both side of \eqref{eq:Asym_opt}. First, we have that for any policy $\pi^\star$, the left-hand side of~(\ref{eq:Asym_opt}) is non-negative since any policy cannot achieve a higher reward compared with that of the optimal policy $\pi^{opt}$, i.e.,
$$\lim_{\rho\rightarrow \infty} \frac{1}{\rho}\Big(R^{\pi^{opt}}( \rho W, \rho N)-R^{\pi^\star}( \rho B, \rho N)\Big)\geq 0.$$ 
Hence we need to show the other direction of \eqref{eq:Asym_opt} that for induced RMI index policy $\pi^\star$, the following holds $$\lim_{\rho\rightarrow \infty} \frac{1}{\rho}\Big(R^{\pi^{opt}}( \rho W, \rho N)-R^{\pi^\star}( \rho B, \rho N)\Big)\leq 0.$$ 
 Let $B_n(s; h)$ be the number of class $n$ arms in state $s$ at time  $h$ and $D_n(s,a;h)$ be the number of class $n$ arms in state $s$ at time $h$ that are being served with action $a\in\mathcal{A}\setminus\{0\}$. 
Similar to \citet{xiong2022AAAI}, by using induction, we show that $\rho\rightarrow\infty$, the following event occurs almost surely,
\begin{align*}
    B_n(s; h)/\rho&\rightarrow P_n(s;h),\\
    D_n(s,a; h)/\rho&\rightarrow P_n(s;h)\xi_n^\star(s,a;h)
\end{align*}
respectively. This leads to the fact that 
\begin{align*}
    \lim_{\rho\rightarrow \infty} \frac{1}{\rho}R^{\pi^\star}(\rho B, \rho N)=\sum_{n=1}^{N}\sum_{h=1}^H\sum\limits_{(s,a)} \mu_n^\star(s,a;h)r_n(s,a), 
\end{align*}
which converges the optimal solution to the LP in Lemma \ref{prop:lp}, 
  which is an upper bound of $\lim_{\rho\rightarrow \infty} \frac{1}{\rho}R^{\pi^{opt}}(\rho B, \rho N)$ due to Lemma \ref{lem:upperbound}.
This completes the proof.

\section{Proofs in Section \ref{sec:analysis} }
In this section, we provide proofs for Lemmas in Section \ref{sec:analysis}.

\subsection{Proof of Lemma \ref{lem:upperbound}}

There exists a set of occupancy measures $\mu_{\pi}^*:=\{\mu^*_n(s,a;h), \forall n\in[N], s\in\cS, a\in\cA, h\in[h]\}$ under policy $\pi$  that optimally solve \eqref{eq:orginal_P} with a relaxed constraint such that  $\mathbb{E}_{\pi}\left[\sum_{n=1}^N A_n^{t,h}\right]\leq B ,h\in[H], t\in[T]$. In addition, $\sum_{t=1}^T \langle{\mu_\pi^*}, r^t\rangle$ is no less than $R(\bs_1, T, \pi^{opt})$, where $\mu_\pi^*$ is the stacked vector of all occupancy measures $\mu_n^*(s,a;h)$ and $r^t$ is the stacked vector of all rewards $r_n^t(s,a)$ for all $n, s, a$.
According to Lemma \ref{prop:lp}, if the "hard" activation constraint in \eqref{eq:orginal_P} is relaxed to an averaged one as  $\mathbb{E}_{\pi}\left[\sum_{n=1}^N A_n^{t,h}\right]\leq B ,h\in[H], t\in[T]$, there is an equivalent LP expressed as in \eqref{eq:pmc_rel_lp}-\eqref{eq:pmc_rel_lp-3}.
To prove Lemme \ref{lem:upperbound},
it is sufficient to show that the relaxed problem achieves no less average reward than the original problem in \eqref{eq:orginal_P}.
The proof is straightforward since the constraints in the relaxed problem  expand the feasible region of the original problem in \eqref{eq:orginal_P}. Denote the feasible region of the original problem  as
\begin{align*}
    \Gamma:=\Bigg\{A_n^{t,h}, \forall h, t\Bigg\vert\sum_{n=1}^{N}A_n^{t,h}\leq B, \forall h\Bigg\},
\end{align*}
and the feasible region of the relaxed problem as 
\begin{align*}
\Gamma^\prime:=\left\{A_n^{t,h}, \forall h,t\Bigg\vert\mathbb{E}_{\pi}\left[\sum_{n=1}^{N}A_n^{t,h}\leq B, \forall h\right]\right\}.
\end{align*}
{It is clear that the relaxed problem expands the feasible region of the original problem in \eqref{eq:orginal_P}, i.e., $\Gamma\subseteq\Gamma^\prime.$  Therefore, the relaxed problem achieves an objective value no less than that of the original problem in \eqref{eq:orginal_P} because the original optimal solution is also inside the relaxed feasibility set \citep{altman1999constrained}, i.e., $\Gamma^\prime$. Denote the optimal occupancy measures of LP in \eqref{eq:pmc_rel_lp}-\eqref{eq:pmc_rel_lp-3} as $\mu_{\pi}^*:=\{\mu^*_n(s,a;h), \forall n\in[N], s\in\cS, a\in\cA, h\in[h]\}$ under a stationary policy $\pi$ induced by $\{\mu^*_n(s,a;h), \forall n\in[N], s\in\cS, a\in\cA, h\in[h]\}$, and hence the maximum reward achieved for the LP in \eqref{eq:pmc_rel_lp}-\eqref{eq:pmc_rel_lp-3} is equal to $\sum_{t=1}^T\langle\mu_{\pi}, r^t\rangle$. Therefore, it follows that $\sum_{t=1}^T\langle\mu_{\pi}, r^t\rangle\geq R(\bs_1, T, \pi^{opt})$, which completes the proof.

\subsection{Proof of Lemma \ref{lemma:regret_decomp}}
According to the definition of regret in \eqref{eq:regret}, we have the following inequality
\begin{align*}
    \Delta(T)&=R(\bs_1, T, \pi^{opt})-R(\bs_1, T, \{\pi^{t}, \forall t\})\allowdisplaybreaks\\
    &= \underset{Term_0\leq 0}{\underbrace{R(\bs_1, T, \pi^{opt})-R(\bs_1, T, \pi^{opt}, \{\hat{r}_n^t, \forall n, t\})}}\allowdisplaybreaks\\
    &\quad+\underset{Term_1}{\underbrace{R(\bs_1, T, \pi^{opt}, \{\hat{r}_n^t, \forall n, t\})-R(\bs_1, T, \{\tilde{\pi}^t, \forall t\}, \{\hat{r}_n^t, \forall n\in[N], t\in[T]\})}}\allowdisplaybreaks\\
    &\quad+\underset{Term_2}{\underbrace{R(\bs_1, T, \{\tilde{\pi}^t, \forall t\}, \{\hat{r}_n^t, \forall n\in[N], t\in[T]\})-R(\bs_1, T, \{\tilde{\pi}^t, \forall t\}, \{{r}_n^t, \forall n\in[N], t\in[T]\})}}\allowdisplaybreaks\\
    &\quad+\underset{Term_3}{\underbrace{R(\bs_1, T, \{\tilde{\pi}^t, \forall t\}, \{{r}_n^t, \forall n\in[N], t\in[T]\})-R(\bs_1, T,\{\pi^t, \forall t\})}}\allowdisplaybreaks\\
    &\leq \underset{Term_1}{\underbrace{R(\bs_1, T, \pi^{\mu^*}, \{\hat{r}_n^t, \forall n, t\})-R(\bs_1, T, \{\tilde{\pi}^t, \forall t\}, \{\hat{r}_n^t, \forall n\in[N], t\in[T]\})}}\allowdisplaybreaks\\
    &\quad+\underset{Term_2}{\underbrace{R(\bs_1, T, \{\tilde{\pi}^t, \forall t\}, \{\hat{r}_n^t, \forall n\in[N], t\in[T]\})-R(\bs_1, T, \{\tilde{\pi}^t, \forall t\}, \{{r}_n^t, \forall n\in[N], t\in[T]\})}}\allowdisplaybreaks\\
    &\quad+\underset{Term_3}{\underbrace{R(\bs_1, T, \{\tilde{\pi}^t, \forall t\}, \{{r}_n^t, \forall n\in[N], t\in[T]\})-R(\bs_1, T,\{\pi^t, \forall t\})}}.
\end{align*}
The first inequality holds due to two reasons. First, $Term_0$ is non-positive as $\{\hat{r}_n^t, \forall n, t\}$ is an  overestimation of the true $r_n^t, \forall n, t$. Second, $R(\bs_1, T, \pi^{\mu^*}, \{\hat{r}_n^t, \forall n, t\})$ is no less than $R(\bs_1, T, \pi^{opt}, \{\hat{r}_n^t, \forall n, t\})$ according to Lemma \ref{lem:upperbound}.

\subsection{Proof of Lemma \ref{lem:term1}}
 Denote the true MDP as ${M}:=\{M_n, \forall n\}$. Hence, we have the following event occurs with high probability when the true transition kernel is inside the confidence ball defined in \eqref{eq:confidence_ball}, i.e.,
\begin{align}\label{eq:ep_set}
 \cE_p^t:&=\{\exists (s, a), n, |P_n^t(s^\prime|s,a)-\hat{P}_n^{t}(s^\prime|s,a)|< \delta_n^{t}(s,a) \}.  
\end{align}
The cumulative probability that the failure events occur is bounded as follows.

\begin{lemma}\label{lemma:failure_sets}
With probability at least $1-2\epsilon$, we have the true transition $P$ is within the confidence set $\cP^t:=\{\cP_n^t, \forall n\in[N]\}$, i.e., event $\cE_p^t$ occurs, 
when $\delta_n^{t}(s,a)\!=\!\sqrt{\frac{1}{2C_n^{t-1}(s,a)}\!\log\!\Big(\frac{4|\cS||\cA|N(t\!-\!1)H}{\epsilon}\Big)}$. 
\end{lemma}
\begin{proof}
By Chernoff-Hoeffding inequality \citep{hoeffding1994probability}, we have
\begin{align*}
    \mathbb{P}\big(|{P}_n(s^\prime|s,a)-\hat{P}_n^t(s^\prime|s,a)|> \delta_n^t(s,a) \big)
\leq \frac{2\epsilon}{|\cS||\cA| N(t-1)H}.
\end{align*}
Using union bound on all states, actions and users, we have 
\begin{align*}
    \mathbb{P}(\mathds{\cE_p^t})&\leq\sum_{n=1}^N\sum_{(s,a)}\mathbb{P}\big(|{P}_n(s^\prime|s,a)-\hat{P}_n(s^\prime|s,a)|> \delta_n^t(s,a)\big)\\
    &\leq \frac{2\epsilon}{(t-1)H}.
\end{align*}
\end{proof}

Next, we have the inequality related with OMD update in Algorithm \ref{alg:importance-policy},  similar to \citet{jin2020learning} as follows. 
\begin{lemma}\label{lem:OMD}
The OMD update with $z_n^1(s,a,s') = \frac{1}{|\cS|^2|\cA|}$ for all $(s,a,s') \in \cS \times \cA \times \cS$, and
$z_n^{t+1} = \arg\max_{z_n\in\Delta^t}\; \eta \langle z_n, \hat{r}_t \rangle - D(z_n \;\|\; z_n^{t})$
 ensures
\begin{align*}
  \sum_{t=1}^T\sum_{n=1}^N\langle z_n - z_n^t, \hat{r}_n^t \rangle\leq \frac{NH\ln(|\cS|^2|\cA|)}{\eta}+\eta{NH(T+1)}. 
\end{align*}
for any $z \in \cap_t \  \Delta^t$, as long as $0\leq\eta\hat{r}_n^t(s,a) \leq 1$ for all $t,n, s,a$.
\end{lemma}

\begin{proof}
Based on \eqref{eq:solution_1st}, we have $\tilde{z}^{t}$ as
\begin{align}\nonumber
\tilde{z}_n^{t}(s, a, s';h) = z_n^{t-1}(s,a,s';h) \exp\left(\eta \hat{r}_n^t(s,a)\right), 
\end{align}
and  $z_n^{t} = \arg\min_{z\in\Delta^t} D(z \;\|\; \tilde{z}_n^{t})$.
According to \eqref{eq:KL_divergence}, we have
\begin{align}\nonumber
  &D(z_n \;\|\; z_n^t) - D(z_n \;\|\; \tilde{z}_n^{t+1}) + D(z_n^t \;\|\;  \tilde{z}_n^{t+1})\allowdisplaybreaks\\
  &=\sum_{s,a,s^\prime,h}z_n(s,a,s^\prime,h)\ln\frac{z_n(s,a,s^\prime,h)}{z_n^{t}(s,a,s^\prime,h)}-z_n(s,a,s^\prime,h)+z_n^{t}(s,a,s^\prime,h)\nonumber\allowdisplaybreaks\\
  &\quad-\sum_{s,a,s^\prime,h}z_n(s,a,s^\prime,h)\ln\frac{z_n(s,a,s^\prime,h)}{\tilde{z}_n^{t+1}(s,a,s^\prime,h)}+z_n(s,a,s^\prime,h)-\tilde{z}_n^{t+1}(s,a,s^\prime,h)\nonumber\allowdisplaybreaks\\
  &\quad+\sum_{s,a,s^\prime,h}z_n^t(s,a,s^\prime,h)\ln\frac{z_n^t(s,a,s^\prime,h)}{\tilde{z}_n^{t+1}(s,a,s^\prime,h)}-z_n^t(s,a,s^\prime,h)+\tilde{z}_n^{t+1}(s,a,s^\prime,h)\nonumber\allowdisplaybreaks\\
  &=\sum_{s,a,s^\prime,h}z_n(s,a,s^\prime,h)\ln\frac{\tilde{z}_n^{t+1}(s,a,s^\prime,h)}{z_n^{t}(s,a,s^\prime,h)}-\sum_{s,a,s^\prime,h}z_n^t(s,a,s^\prime,h)\ln\frac{\tilde{z}_n^{t+1}(s,a,s^\prime,h)}{z_n^{t}(s,a,s^\prime,h)}\nonumber\allowdisplaybreaks\\
  &=\eta\langle z_n - z_n^t, \hat{r}_n^t \rangle,
\end{align}
where the last inequality is due to \eqref{eq:solution_1st}. Furthermore, according to \eqref{eq:solution2} generalized Pythagorean theorem,  we have $ D(z_n \;\|\; {z}_n^{t+1})\leq D(z_n \;\|\; \tilde{z}_n^{t+1})$. Hence, the following inequality holds
\begin{align}\label{eq: lem8} \eta\sum_{t=1}^T\sum_{n=1}^N\langle z_n - z_n^t, \hat{r}_n^t \rangle&\leq \sum_{t=1}^T\sum_{n=1}^N D(z_n \;\|\; z_n^t) - D(z_n \;\|\; {z}_n^{t+1}) + D(z_n^t \;\|\;  \tilde{z}_n^{t+1})\allowdisplaybreaks\nonumber\\
&=\sum_{n=1}^N D(z_n \;\|\; z_n^1) - D(z_n \;\|\; {z}_n^{T+1})+\sum_{t=1}^T\sum_{n=1}^N D(z_n^t \;\|\;  \tilde{z}_n^{t+1})\nonumber\allowdisplaybreaks\\
&= \underset{(a_1)}{\underbrace{\sum_{n=1}^N \sum_{s,a,s^\prime,h}z_n(s,a,s^\prime,h)\ln\frac{z_n^{T+1}(s,a,s^\prime,h)}{z_n^{1}(s,a,s^\prime,h)}}}\nonumber\allowdisplaybreaks\\
&+\underset{(a_2)}{\underbrace{\sum_{n=1}^N \sum_{s,a,s^\prime,h}z_n^1(s,a,s^\prime,h)-z_n^{T+1}(s,a,s^\prime,h)}}+\underset{(a_3)}{\underbrace{\sum_{t=1}^T\sum_{n=1}^N D(z_n^t \;\|\;  \tilde{z}_n^{t+1})}}.
\end{align}
Next, we bound each term in \eqref{eq: lem8}. $(a_1)$ is bounded as
\begin{align*}
    (a_1)\leq \sum_{n=1}^N \sum_{s,a,s^\prime,h} z_n(s,a,s^\prime,h)\ln(|\cS|^2|\cA|)\leq NH\ln(|\cS|^2|\cA|),
\end{align*}
due to the initialization of $z_n^1(s,a,s^\prime,h), \forall s,a,h$. Similarly, $(a_2)$ is bounded as $(a_2)\leq NH$.
$(a_3)$ is bounded as follows,
\begin{align*}
D(z_n^t \;\|\; \tilde{z}_n^{t+1})
&= \sum_{s,a,s^\prime,h}z_n^t(s,a,s^\prime,h)\ln\frac{z_n(s,a,s^\prime,h)}{\tilde{z}_n^{t+1}(s,a,s^\prime,h)}-z_n(s,a,s^\prime,h)+\tilde{z}_n^{t+1}(s,a,s^\prime,h) \\
&= \sum_{s,a,s^\prime, h}-\eta\hat{r}_n^t(s,a)z_n^t(s,a,s^\prime,h)-z_n^t(s,a,s^\prime,h)+z_n^t(s,a,s^\prime,h)\exp(\eta\hat{r}_n^t(s,a))\\
&\leq \eta^2\sum_{s,a,s^\prime, h} z_n^t(s,a,s^\prime,h)\leq H\eta^2,
\end{align*}
where the first inequality is due to the fact $e^{z} - 1- z\leq z^2$ for all $z \in [0, 1]$, and the second inequality is due to the property of occupancy measure.
Substituting $(a_1)-(a_3)$ back to \eqref{eq: lem8}, we have
\begin{align*}
\sum_{t=1}^T\sum_{n=1}^N\langle z_n - z_n^t, \hat{r}_n^t \rangle\leq \frac{NH\ln(|\cS|^2|\cA|)}{\eta}+\eta{NH(T+1)}. 
\end{align*}
This completes the proof. 
\end{proof}

According to the definition of $\pi^{\mu^*}$ and $\tilde{\pi}^t$ and the fact that they are working on the problem with relaxed activation constraint, we can transform $Term_1$ into the following form:
\begin{align*}
    Term_1&=R(\bs_1, T, \pi^{\mu^*}, \{\hat{r}_n^t, \forall n, t\})-R(\bs_1, T, \{\tilde{\pi}^t, \forall t\}, \{\hat{r}_n^t, \forall n\in[N], t\in[T]\})\\
    &=\sum_{t=1}^T\sum_{n=1}^N\langle \mu_n^*, \hat{r}_n^t\rangle-\sum_{t=1}^T\sum_{n=1}^N\langle z_n^t, \hat{r}_n^t\rangle\\
    &=\sum_{t=1}^T\sum_{n=1}^N\langle \mu_n^* - z_n^t, \hat{r}_n^t \rangle.
\end{align*}
Since with probability $1-2\epsilon$, we have $\mu_n^*\in \cap_t \cZ^t$ according to Lemma \ref{lemma:failure_sets}. Hence, Lemma \ref{lem:term1} directly follows Lemma \ref{lem:OMD}.

\subsection{Proof of Lemma \ref{lem:term2}}

According to the definition, we can rewrite $Term_2$ as
\begin{align*}
    Term_2&=R(\bs_1, T, \{\tilde{\pi}^t, \forall t\}, \{\hat{r}_n^t, \forall n\in[N], t\in[T]\})-R(\bs_1, T, \{\tilde{\pi}^t, \forall t\}, \{{r}_n^t, \forall n\in[N], t\in[T]\})\\
    &=\sum_{t=1}^T\sum_{n=1}^N\langle  z_n^t, \hat{r}_n^t-r_n^t \rangle.
\end{align*}
The key is to characterize the gap between $\hat{r}_n^t-r_n^t$. We first present the widely used Hoeffding inequality in the following lemma.
\begin{lemma}[Hoeffding inequality \citep{bercu2015concentration}]
Let $X_1, X_2,\ldots$ be independent random variables with $b\leq|X_i|\leq c$ for all $i$. Define $S_n=X_1+\ldots+X_n$. Then for all $\epsilon>0$
\begin{align*}
    Pr\left(S_n-\mathbb{E}[S_n]>\epsilon\right)\leq \exp\left(-\frac{\epsilon^2}{2n(c-b)^2}\right).
\end{align*}
\label{lem:hoeffding}
\end{lemma}

Notice that $\langle z_n^t, \hat{r}_n^t\rangle\leq HN$ due to the overestimation of $\hat{r}_n^t$ in  \eqref{eq:reward_estimator}, according the Hoeffding inequality in Lemma \ref{lem:hoeffding}, we have that with probablity at least $1-\epsilon$, such taht 
\begin{align*}
    \sum_{t=1}^T\sum_{n=1}^N \langle z_n^t, \hat{r}_n^t-\mathbb{E}[\hat{r}_n^t]\rangle\leq HN\sqrt{TN\ln1/\epsilon}.
\end{align*}
Thus, we have at least probability $1-3\epsilon$ that 
\begin{align}\label{eq:term2}
    Term_2&\leq \sum_{t=1}^T\sum_{n=1}^N\langle  z_n^t, \mathbb{E}[\hat{r}_n^t]-r_n^t \rangle+HN\sqrt{TN\ln 1/\epsilon}\leq \sum_{t=1}^T\sum_{n=1}^N\langle  z_n^t, \delta_n^t \rangle+HN\sqrt{TN\ln 1/\epsilon},
\end{align}
where the inequality holds due to the definition of $\hat{r}_n^t(s,a)$ and the fact that
$\frac{r_n(s,a)\mathds{1}(\exists h, S_n^{t}(h)=\!s, A_n^{t}(h)=\!a)}{\max\{c_n^{t}(s,a),1\}/H}$ is an unbiased estimator of $r_n^t(s,a)$.

Next, we bound $\sum_{t=1}^T\sum_{n=1}^N\langle  z_n^t, \delta_n^t \rangle$ in the following lemma.
\begin{lemma}\label{lem:10}
The following inequality holds
\begin{align*}
   \sum_{t=1}^T\sum_{n=1}^N\langle  z_n^t, \delta_n^t \rangle \leq
\sqrt{2}H\sqrt{\ln\frac{4SANT}{\epsilon}}\cdot \sqrt{NT}.
\end{align*}
\end{lemma}

\begin{proof}
The proof goes as follows.
\begin{align*}
\sum_{t=1}^T\sum_{n=1}^N\langle  z_n^t, \delta_n^t \rangle 
&\overset{(a)}{\leq} H\sum_{t=1}^T\sum_{n=1}^N \sum_{s,a}\delta_n^t(s,a)\displaybreak[3]\\
&\leq H\sum_{t=1}^{T}\sum_{n=1}^N\sum_{s,a} \sqrt{\frac{1}{2C_n^{t}(s,a)}\ln\frac{4SANT}{\epsilon}}\allowdisplaybreaks\\
&{\leq}H\sqrt{\ln\frac{4SANT}{\epsilon}}\sum_{n=1}^N\sum_{(s,a)} \sqrt{{C_n^{T}(s,a)}}\allowdisplaybreaks\\
&\overset{(b)}{\leq} \sqrt{2}H\sqrt{\ln\frac{4SANT}{\epsilon}}\sum_{n=1}^N \sqrt{\sum_{(s,a)}{C_n^{T}(s,a)}}\allowdisplaybreaks\\
&\overset{(c)}{\leq} \sqrt{2}H\sqrt{\ln\frac{4SANT}{\epsilon}}\cdot \sqrt{NT},
\end{align*}
where (a) follows since $z_n^t$ is a probability measure, (b) follows Cauchy-Schwartz inequality and (c) uses the fact that $\sum_{n=1}^N\sum_{(s,a)}C_n^T(s,a)\leq NT.$ This completes the proof.
\end{proof}

\textbf{Bound on $Term_2$.} Combining the results in \eqref{eq:term2} and Lemma \ref{lem:10}, we can bound $Term_2$ as
\begin{align*}
    Term_2\leq \sqrt{2}H\sqrt{\ln\frac{4SANT}{\epsilon}}\cdot \sqrt{NT}+HN\sqrt{TN\ln 1/\epsilon}.
\end{align*}

\subsection{Proof of Lemma \ref{lem:term3}}\label{sec:thm_regret_proof}

According to Lemma \ref{lemma:failure_sets}, with probability at least $1-2\epsilon$,  the true transition kernel $\cP$ is inside the confidence ball as defined in \eqref{eq:ep_set}. Due to the optimism of the confidence ball, in each episode, we have that $\tilde{\pi}^t$ under an optimisitic MDP with transition $\tilde{\cP}$ achieves no worse performance than $\pi^t$ under the true MDP with transition $\cP$. Hence, we decompose the regret in $Term_3$ into episodic regrets.  
For simplicity, we denote 
$c_n^t(s,a):=\sum_{h=1}^{H}\mathds{1}(S_n^{t,h}=s,A_n^{t,h}=a)$  as the state-action counts for $(s,a)$ in episode $t$, and $\gamma^{t,*}$ as the average reward achieved per decision epoch by policy $\tilde{\pi}^t$. Then, we define the regret during episode $t$ as follows,
\begin{align}\label{eq:reg_epis}
\Delta_t:=H\gamma^{t,*}-\sum_{(s,a)}\sum_n c_n^t(s,a)r_n^t(s,a).
\end{align}
The relation between the total regret in $Term_3$ and the episodic regrets $\Delta_k$ is as follows.  
\begin{lemma}
The regret in $Term_3$ is upper-bounded by
\begin{align}\label{eq:regret-upperbound}
Term_3\leq \sum\limits_{t=1}^{T} \Delta_t+ \sqrt{\frac{1}{4}T\log\Big(\frac{|\cS||\cA|NT}{\epsilon}\Big)}, 
\end{align}
with probability at least $1-\epsilon$. 
\end{lemma}

\begin{proof}
Using Azuma-Hoeffding's inequality, we have
\begin{align}\nonumber
\mathbb{P}&\Bigg(R(\bs_1, T,\{\pi^t, \forall t\})\leq \sum_{t=1}^T\sum\limits_n\sum\limits_{(s,a)} c_n^t(s,a)r^t_n(s,a)
-\sqrt{\frac{1}{4}T\log\Big(\frac{|\cS||\cA|NT}{\epsilon}\Big)}\Bigg)\\
&\qquad\leq \left(\frac{\epsilon}{|\cS||\cA|NT}\right)^{1/2}. 
\end{align}
Therefore,
\begin{align*}
\Delta(T)&=\sum_{t=1}^TH\gamma^{t,*}-R(\bs_1, T,\{\pi^t, \forall t\})\displaybreak[0]\\
&\leq \sum_{t=1}^TH\gamma^{t,*}-\sum_{t=1}^T\sum\limits_n\sum\limits_{(s,a)} c_n^t(s,a)r^t_n(s,a)
+ \sqrt{\frac{1}{4}T\log\Big(\frac{|\cS||\cA|NT}{\epsilon}\Big)}\displaybreak[1]\\
&=\sum\limits_{k=1}^{K} \Delta_k+ \sqrt{\frac{1}{4}T\log(\frac{|\cS||\cA|NT}{\epsilon})}.
\end{align*}
This completes the proof.
\end{proof}

 The following lemma characterizes the regret under scenario where the true transition $\cP$ is outside of the confidence ball, which occurs at a small probability at most $2\epsilon$. 
\begin{lemma}\label{lemma:regret-failure-events}
We have $$\sum_{t=1}^{T}\Delta_t\mathds{1}(\cP\notin \cE_p^t)\leq B\sqrt{TH}.$$ 
  
\end{lemma}
\begin{proof}
 By Lemma \ref{lemma:failure_sets}, we have
 \begin{align*}
     \sum_{t=1}^{T}\Delta_t\mathds{1}(\cP\notin \cE_p^t)&\leq \sum_{t=1}^{T}\sum_n\sum_{(s,a)}c_n^t(s,a)\mathds{1}(\cP\notin \cE_p^t)\displaybreak[0]\\
&\leq \sum_{t=1}^{T} HB \mathds{1}(\cP\notin \cE_p^k)\displaybreak[1]\\
&\leq B\sum_{i=1}^{TH}i\mathds{1}(\cP\notin \cE_p^t)\leq B\sqrt{TH}.
 \end{align*}

\end{proof}

We next bound the regrets $\Delta_k$ when the true transition lies in the confidence ball, i.e., event $\cE_p^t$ occurs.

\begin{lemma}\label{lemma_good_events}
Under event $\cE_p^t$, we have 
\begin{align}\nonumber
\Delta_t\mathds{1}(\cP\in\cE_p^t) &\leq 
\sum_{(s,a)}\sum_nc_n^t(s,a)(\gamma^{t,*}/B-r_n^t(s,a))\displaybreak[0]\\
&\qquad+\sum_{(s,a)}\sum_nc_n^t(s,a)\kappa_1,
\end{align}
where $\kappa_1>0$ is as in Lemma~\ref{lemma:kappa} and $\tilde{\gamma}^t$ is the average reward achieved by an optimistic policy induced from the confidence ball (i.e., $\cE_p^t$) when the rewards is shifted by $2\delta_n^t(s,a), \forall s,a,n$, i.e.,  $\tilde{r}^t_n(s,a):=r^t_n(s,a)+2\delta_n^t(s,a)$.
\end{lemma}

\begin{proof}
This result follows directly from the definition of episodic regret in \eqref{eq:reg_epis} and the fact that the total activated arms count should be less than $HB$ per episode.
\end{proof}

\begin{lemma}\label{lemma:kappa}
With probability at least $1-2\epsilon$, we have 
$\tilde{\gamma}^t\geq \gamma^{t,*}-\kappa_1,$
where $\kappa_1=\cO(\frac{c_1}{\sqrt{tH}})$ with $0<c_1<1$.
\end{lemma}
\begin{proof}
We denote by $S_h^t$ and 
$S_h^{t,*}$ the arms that are activated by an optimistic policy and  $\tilde{\pi}^t$ under true reward and at the $h$-th decision epoch of episode $t$. For abuse of notation, we use $\tilde{r}_n^t(h)$ to denote the reward of a specific state-action pair $(s,a)$ visited at $h$ for arm $n$. Due to optimism, we have 
\begin{align}\nonumber
\sum_{n\in S_h^t\setminus S_h^{t,*}}\tilde{r}^t_n(h)\geq \sum_{n\in S_h^{t,\star}\setminus S_h^t}\tilde{r}^t_n(h).
\end{align}
When conditioned on the event $\cE_p^t$ simultaneously, we have
\begin{align*}
\sum_{n\in S_h^k\setminus S_h^{k,*}}r_n^t(h)+2\delta_n^k&\geq\sum_{n\in S_h^k\setminus S_h^{k,*}}\tilde{r}^t_n(h)\geq \sum_{n\in S_h^{k,*}\setminus S_h^k}\tilde{r}^t_n(h)\\ &\geq\sum_{n\in S_h^{k,*}\setminus S_h^k}r_n^t(h).
\end{align*}
Hence, we have
\begin{align}\nonumber
\sum_{n\in S_h^{t,(*)}\setminus S_h^t}r_n^t(h)-\sum_{n\in S_h^t\setminus S_h^{t,*}}r^t_n(h)\leq \sum_{n\in S_h^t\setminus S_h^{t,*}}2\delta_n^t\leq 2B\delta^t,
\end{align}
where $\delta_n^t=\delta^t, \forall n$ and is given in Lemma~\ref{lemma:failure_sets}. Therefore, we have
$\gamma^{t,*}-\tilde{\gamma}^t\leq {2B\delta}=\cO\Big(\frac{2B}{\sqrt{tH}}\Big).$
\end{proof}

Upon combining the results obtained in Lemma~\ref{lemma_good_events} and Lemma~\ref{lemma:kappa}, we obtain the following.
\begin{lemma}\label{lemma:regret-good-events}
We have $$\sum_{t=1}^{T}\Delta_t\mathds{1}(\cP\in \cE_p^t)\leq \Big(\sqrt{2\log\Big(\frac{4|\cS||\cA|NTH}{\epsilon}\Big)}+2B\Big)\sqrt{|\cS||\cA|NTH}.$$
\end{lemma}

\begin{proof} 
From Lemma~\ref{lemma_good_events} and Lemma~\ref{lemma:kappa}, we can rewrite the summation over $\Delta_t$ as follows:
\begin{align*}
&\sum_{t=1}^{T}\Delta_k\mathds{1}(\cP\!\in\!\cE_p^t)\!\leq\! \underset{\text(I)}{\underbrace{\sum_{t=1}^{T}\sum_{(s,a)}\sum_nc_n^t(s,a)(\gamma^{t,*}/B\!-\!\tilde{r}_n^t(s,a))}}\displaybreak[0]\\
&\qquad+\underset{\text{(II)}}{\underbrace{\sum_{t=1}^{T}\sum_{(s,a)}\sum_nc_n^t(s,a)\Big(\tilde{r}_n^t(s,a)-r_n^t(s,a)+\frac{1}{\sqrt{tH}}\Big)}}.
\label{eq:sum_regret}
\end{align*}
It is clear that $\text{(I)}$ is non-positive. Conditioned on good events, it is clear that 
$\tilde{r}_n^t(s,a)-{r}_n^t(s,a)\leq 2\delta_n^t(s,a).$
Therefore, we have
\begin{align*}
\sum_{t=1}^{T}\Delta_t\mathds{1}(\cP\in \cE_p^k)
&\leq \Bigg(\sqrt{2\log\Big(\frac{4|\cS||\cA|NTH}{\epsilon}\Big)}+2B\Bigg)\sum_{t=1}^T\sum_{n=1}^N\sum_{(s,a)}\frac{c_n^t(s,a)}{\sqrt{C_n^t(s,a)}}\displaybreak[1]\\
    &\leq \Bigg(\sqrt{2\log\Big(\frac{4|\cS||\cA|NTH}{\epsilon}\Big)}+2B\Bigg)\sqrt{|\cS||\cA|NTH},
\end{align*}
where last inequality is due to Jensen's inequality and Lemma \ref{lemma:coverc}.

\begin{lemma}\label{lemma:coverc}
    For any sequence of numbers $w_1,w_2,...,w_T$ with $0\leq w_k$, define $W_{k}:=\sum_{i=1}^k w_i$,
    \begin{align*}
        \sum_{k=1}^T \frac{w_k}{\sqrt{W_{k}}} \leq (\sqrt{2} +1)\sqrt{W_T}.
    \end{align*}
\end{lemma}
\begin{proof}
The proof follows by induction. 
When $t=1$, it is true as $1\leq\sqrt{2}+1$.
Assume for all $k\leq t-1$, the inequality holds, then we have the following:
\begin{align*}
    &\sum_{k=1}^T \frac{w_k}{\sqrt{W_{k}}}= \sum_{k=1}^{T-1} \frac{w_k}{\sqrt{W_{k}}} + \frac{w_T}{\sqrt{W_{T}}} \displaybreak[0]\\
    &\leq (\sqrt{2} +1) \sqrt{W_{T-1}} + \frac{w_T}{\sqrt{W_{T}}} \displaybreak[0]\\
    &=\sqrt{{(\sqrt{2}+1)}^2 W_{T-1} + 2(\sqrt{2}+1)w_T \sqrt{\frac{W_{T-1}}{W_T}} + \frac{{w_T}^2}{W_T}} \displaybreak[0]\\
    &\leq \sqrt{{(\sqrt{2}+1)}^2 W_{T-1} + 2(\sqrt{2}+1)w_T \sqrt{\frac{W_{T-1}}{W_{T-1}}} + \frac{{w_T}W_T}{W_T}} \displaybreak[0]\\
    & = \sqrt{{(\sqrt{2}+1)}^2 W_{T-1} + (2(\sqrt{2}+1)+1)w_T} \displaybreak[0]\\
    & = (\sqrt{2}+1)\sqrt{(W_{T-1}+ w_T)}\displaybreak[0]\\
    & = (\sqrt{2}+1)\sqrt{W_T}.
\end{align*}
\end{proof}

\end{proof}

\section{Details of the Numerical Case-Study}\label{sec:experiments}

\textbf{Continuous Positive Airway Pressure Therapy (CPAP).} 
The CPAP \citep{kang2013markov,herlihy2023planning,li2022towards, wang2024online} is a highly effective treatment when it is used
consistently during sleeping for adults with obstructive
sleep apnea. Similar non-adherence to CPAP in patients hinders the effectiveness, we adapt the Markov model of CPAP
adherence behavior to a two-state system with the clinical adherence criteria. To elaborate, three distinct states are defined to characterize adherence levels: low, intermediate, and acceptable. Patients are categorized into two clusters: ``Adherence" and ``Non-Adherence." Those within the ``Adherence" cluster exhibit a greater likelihood of maintaining an acceptable level of adherence. In particular,  
the state space is ${1,2,3}$, which represents low, intermediate and acceptable adherence level respectively.  The difference between these two groups reflects on their transition probabilities, as in Figures.  \ref{Fig:CPAPtransition1}- \ref{Fig:CPAPtransition2}. Generally speaking, the first group has a higher probability of staying in a good adherence level. From each group, we construct 10 arms, whose transition probability matrices are generated by adding a small noise to the original one. Actions such as text patients/ making a call/ visit in person will cause a 5\% to 50\% increase in adherence level. The budget is set to $B = 10$.  The objective is to maximize the total adherence level.

In standard CPAP, which is used for stochastic RMAB setting, the reward is set as the $1$ for state ``low adherence", $2$ for state ``intermediate adherence", and $3$ for state ``high adherence". Based on this, we randomly generate a sequence of coefficients to increase or decrease the reward for each episode.   In our setting, we consider an MDP with $50$ episode, and each episode contains $50$ time steps. The control coefficient for episode $i$ is $0.5+(i-1)/49$. The whole experiment runs in $1000$ Monte Carlo independent rounds.  
\vspace{-0.5cm}
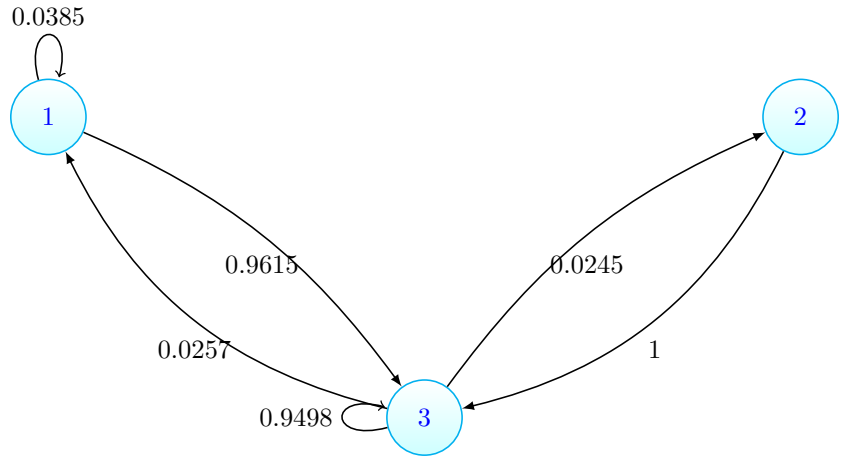
\begin{figure}[ht]
\begin {center}
\begin {tikzpicture}[-latex ,auto ,node distance =4 cm and 5cm ,on grid ,
semithick ,
state/.style ={ circle ,top color =white , bottom color = processblue!20 ,
draw,processblue , text=blue , minimum width =1 cm}] x
\node[state] (C)
{$3$};
\node[state] (A) [above left=of C] {$1$};
\node[state] (B) [above right =of C] {$2$};
\path (A) edge [loop above] node[above] {$0.0385$} (A);
\path (C) edge [bend left =25] node[below =0.15 cm] {$0.0257$} (A);
\path (A) edge [bend right = -15] node[below =0.15 cm] {$0.9615$} (C);
\path (C) edge [loop left] node[left] {$0.9498$} (C);
\path (C) edge [bend left =15] node[below =0.15 cm] {$0.0245$} (B);
\path (B) edge [bend right = -25] node[below =0.15 cm] {$1$} (C);
\end{tikzpicture}
\caption{Transition diagram for CPAP Cluster 1}
\label{Fig:CPAPtransition1}
\end{center}
\end{figure}
\vspace{-0.5cm}
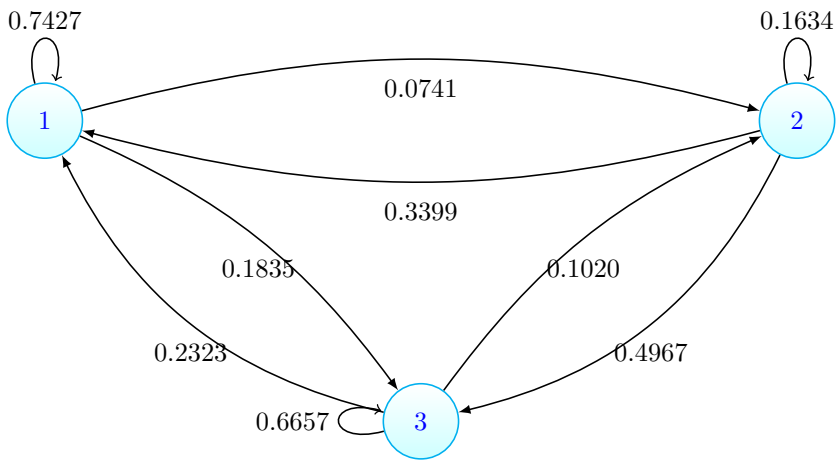
\begin{figure}[ht]
\begin {center}
\begin {tikzpicture}[-latex ,auto ,node distance =4 cm and 5cm ,on grid ,
semithick ,
state/.style ={ circle ,top color =white , bottom color = processblue!20 ,
draw,processblue , text=blue , minimum width =1 cm}] x
\node[state] (C)
{$3$};
\node[state] (A) [above left=of C] {$1$};
\node[state] (B) [above right =of C] {$2$};
\path (A) edge [loop above] node[above] {$0.7427$} (A);
\path (C) edge [bend left =25] node[below =0.15 cm] {$0.2323$} (A);
\path (A) edge [bend right = -15] node[below =0.15 cm] {$0.1835$} (C);
\path (A) edge [bend right = -15] node[below =0.15 cm] {$0.0741$} (B);
\path (C) edge [loop left] node[left] {$0.6657$} (C);
\path (C) edge [bend left =15] node[below =0.15 cm] {$0.1020$} (B);
\path (B) edge [bend right = -25] node[below =0.15 cm] {$0.4967$} (C);
\path (B) edge [bend left = 15] node[below =0.15 cm] {$0.3399$} (A);
\path (B) edge [loop above] node[above] {$0.1634$} (B);
\end{tikzpicture}
\caption{Transition diagram for CPAP Cluster 2}
\label{Fig:CPAPtransition2}
\end{center}
\end{figure}

\textbf{A Deadline Scheduling Problem.} 
We consider the  deadline scheduling problem for the scheduling of electrical
vehicle charging stations.  A charging station (agent) has total $N$ charging spots (arms) and can charge $M$ vehicles in each round. The charging station
obtains a reward for each unit of electricity that it provides to a vehicle and receives a penalty (negative reward) when a vehicle is not fully charged. The
goal of the station is to maximize its net reward. We use exactly the same
setting as in (Xiong et al. 2022a) for our experiment. 
More specifically, the state of an arm is denoted by a pair of integers $(D; B)$, where $B$ is the amount of
electricity that the vehicle still needs and $D$ is the time until the vehicle leaves the station.  When a
charging spot is available, its state is $(0; 0)$. $B$ and $D$ are upper-bounded by $9$ and 
$12$, respectively.  Hence, the size of state space is $109$ for each arm.
The state transition is given by
\begin{align}\nonumber
    S_i(t+1)=\begin{cases}
    (D_i(t)-1, B_i(t)-a_i(t)), \quad\text{if $D_i(t)>1$},\\
    (D, B), \quad\text{with prob. 0.7 if $D_i(t)\leq 1$},
    \end{cases}
\end{align}
where $(D, B)$ is a random state when a new vehicle arrives at the charging spot $i$.
Specifically, $a_i(t)=0$ means being passive and $a_i(t)=1$ means being active.   There are total $N=100$ charging spots and a maximum $M=30$ can be served at each time.

The agent receives a base reward at each time from arm $i$ according to
\begin{align*}
    r_i(t)=\begin{cases}
    (1-0.5)a_i(t), \quad \text{if $B_i(t)>0, D_i(t)>1$},\\
    (1-0.5)a_i(t)-0.2(B_i(t)-a_i(t))^2, \\
    \qquad\qquad\qquad\quad \text{if $ B_i(t)>0, D_i(t)=1$},\\
    0, \quad \text{Otherwise}.
    \end{cases}
\end{align*}
Based on this, we randomly generate a sequence of coefficients
to increase or decrease the reward for each episode. In our setting, we consider an MDP with 100 episode, and each episode
contains 100 time steps. The control coefficient for episode i is $0.5+(i-1)/99$. The whole experiment runs in 1000 Monte
Carlo independent rounds.


\end{document}